%% file: main.tex
\documentclass{article}
\usepackage[T1]{fontenc}
\usepackage{iclr2027_conference,times}
\usepackage{amsmath,amssymb,amsthm,mathtools}
\usepackage{booktabs,multirow,array}
\usepackage{longtable}
\usepackage{float}
\usepackage{placeins}
\usepackage{graphicx}
\usepackage{xcolor}
\usepackage[utf8]{inputenc}
\usepackage[hidelinks]{hyperref}
\usepackage{url}
\usepackage{algorithm}
\usepackage{algpseudocode}
\usepackage{enumitem}
\hypersetup{pdftitle={The Price of Token Boundaries: Compression Certificates and Prediction},
  pdfauthor={Yuhao DU, Shunian Chen}}

\newtheorem{theorem}{Theorem}
\newtheorem{proposition}[theorem]{Proposition}
\newtheorem{corollary}[theorem]{Corollary}
\newtheorem{lemma}[theorem]{Lemma}
\theoremstyle{definition}

\newtheorem{remark}[theorem]{Remark}

\newcommand{\OPT}{\mathrm{OPT}}
\newcommand{\SP}{\mathrm{SP}}

\newcommand{\slot}{\textsc{SlotPrice}}

\title{The Price of Token Boundaries:\\
Compression Certificates and Prediction}

\author{\parbox[t]{\dimexpr\textwidth-2\tabcolsep\relax}{Yuhao DU$^{1,2}$, Shunian Chen$^{1}$\\\normalfont $^{1}$The Chinese University of Hong Kong, Shenzhen\\\normalfont $^{2}$Shenzhen Loop Area Institute\\\normalfont\texttt{yuhaodu1@link.cuhk.edu.cn}}}

\input{rebuild_numbers}
\input{validation_numbers}
\input{supplementary_numbers}

\input{replication_numbers}

\input{boundary_numbers}

\iclrfinalcopy
\begin{document}
\maketitle
\lhead{Preprint}

\begin{abstract}
Pre-tokenisation restricts which text fragments can become prediction units, but its
compression cost is obscured when tokenisers are compared only under the same boundaries.
We measure this cost by bounding the minimum token count from both sides, with and without a
regular-expression boundary rule. Non-negative prices on token occurrences yield a lower
bound through shortest paths and vocabulary-budget selection; maximising over all prices
recovers the linear-programming relaxation, and an independent integer checker certifies
the reported values. On English Wikipedia, boundaries increase the optimal token count by
$\dattaxEnLo$--$\dattaxEnHi\%$. Byte pair encoding lies $\datbpeGapEnRegex\%$ above the
constrained lower bound, but $\datbpeGapEnNone\%$ above the unrestricted bound.
Compression and prediction favour different dictionaries: at $85$M non-embedding parameters
and matched training-token budgets, unrestricted fitting yields higher mean held-out bits
per byte under a common unrestricted decoder in all $\datTestN$ languages in the paired study
and $\datRepTokenPositiveN$ of $\datRepLangN$ under independent tuning and evaluation.
To study intermediate boundary policies, we introduce
\emph{boundary licences}, which limit the vocabulary entries permitted to cross cuts and
admit the same form of certificate. On separate English and Chinese fitting corpora,
licensing $10\%$ of the vocabulary budget recovers $\boundaryEnTenRecovery\%$ and
$\boundaryZhTenRecovery\%$ of the achieved token-count reduction from removing all cuts.
These results quantify the compression cost of boundaries while separating it from the
prediction quality of the resulting token units.
\end{abstract}

\section{Introduction}

Tokenisers split text into the pieces a language model (LM) predicts over. A regular expression (regex) first divides text into
pieces; vocabulary training then selects recurring substrings within those pieces.
Preventing a token from spanning a word boundary can increase the number of prediction
steps. It may also preserve useful prediction units. These two effects are usually assessed
separately: compression methods seek shorter sequences, while language-model experiments
compare the quality of the resulting token streams. We ask what the boundary constraint
costs at the compression optimum, and what is gained by keeping it during vocabulary training
(Figure~\ref{fig:overview}).

\begin{figure}[t]
\centering
\includegraphics[width=\textwidth]{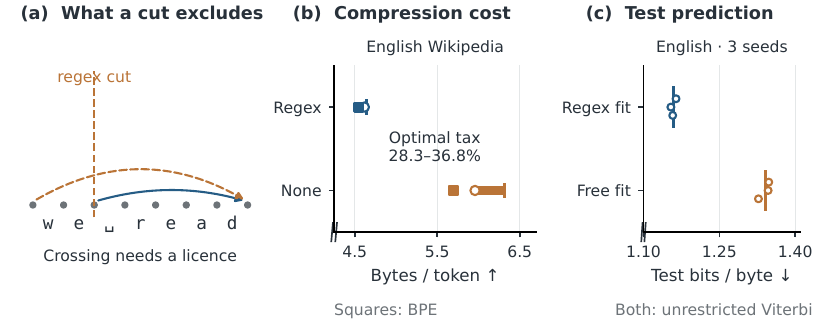}
\caption{\textbf{A compression advantage need not improve prediction.}
(a) A cut removes crossing candidate occurrences; a licence lets a token cross it.
(b) English compression optima, enclosed by feasible vocabularies and integer-checked
lower bounds ($K{=}32768$, $L{=}16$, Wikipedia); the value axis is broken. Squares mark
BPE's achieved compression ($\datbpeEnRegex$ and $\datbpeEnNone$ bytes per token);
the annotated tax interval is $\dattaxEnLo$--$\dattaxEnHi\%$.
(c) English test bit rates from three paired LM seeds; the value axis is broken.
Both arms use unrestricted Viterbi encoding, with dictionaries fitted with or without
regex boundaries. The unrestricted fitted dictionary has higher bit rate in each seed.
In panels (b) and (c), blue marks the regex-constrained arm and gold the unrestricted arm.
Compression certificates and test losses refer to their respective evaluation corpora.}
\label{fig:overview}
\end{figure}

A useful puzzle motivates the first question. ConvexTok and
JOLT~\citep{convextok,jolt} report that byte pair encoding (BPE) is close to a compression lower bound under
pre-tokenisation (minimum-count for ConvexTok, greedy longest-match for JOLT). SuperBPE and BoundlessBPE~\citep{superbpe,boundlessbpe} nevertheless
obtain substantially shorter sequences by relaxing token boundaries. The observations
concern different optima. A certificate inside the regex-constrained feasible set cannot
bound the cost of the constraint itself. Measuring that cost requires a certificate that
also covers unrestricted documents, where the number of candidate occurrences grows as
corpus bytes times maximum token length.

We obtain such a certificate by assigning a non-negative price to each candidate occurrence.
The lower bound is a shortest-path cost minus the largest $K$ aggregate token bids,
where $K$ is the vocabulary budget.
This specialises classical Lagrangian relaxation~\citep{geoffrion1974,fisher1981} to a
computation that needs no general-purpose linear-programming (LP) solver. Prices may be optimised approximately;
validity is checked separately in integer arithmetic. A feasible vocabulary supplies the
other end of the interval, so the cost of pre-tokenisation can be measured independently
of whether BPE or our vocabulary search reaches the optimum.

In a paired $\datTestN$-language study,
unrestricted fitted dictionaries compress better but have higher mean test bits per byte
in every language. A replication with independent tuning and evaluation seeds retains
this direction in $\datRepTokenPositiveN$ of $\datRepLangN$ languages. This extends the evidence that minimum token
count is an incomplete objective for language modelling~\citep{pathpiece,tokeval2026}.
Appendix~\ref{app:related} gives an extended discussion of related work and scope.

\begin{samepage}
Vocabulary content, token frequencies and effective byte context change together in this
comparison. Boundary licences control a complementary quantity: the number of vocabulary
entries permitted across regex cuts. Their endpoints reproduce the constrained and unrestricted optimisation problems, and their
certificate measures the compression available at intermediate budgets.\par
\end{samepage}

\paragraph{Contributions.}
We \textbf{certify the cost} of boundaries with an exact numerical checker and two-sided
optimal-tax intervals (\S\ref{sec:theory}--\ref{sec:tax});
\textbf{measure prediction differences} using held-out test comparisons with an explicit
vocabulary and decoder protocol (\S\ref{sec:lm}); and
\textbf{control boundary permissions} with a boundary-licence formulation and measured compression
frontiers (\S\ref{sec:frontier}).

\section{Related work}
\label{sec:related}

\paragraph{Vocabulary optimisation.}
BPE and Unigram optimise merge-frequency and likelihood objectives,
respectively~\citep{gage1994,sennrich2016,kudo2018,sentencepiece}.
BytePiece~\citep{bytepiece2023} trains byte-level Unigram vocabularies with a byte n-gram
model and decodes by maximising summed token log probabilities.
ConvexTok~\citep{convextok} formulates vocabulary selection and minimum-count segmentation
as a linear programme; JOLT~\citep{jolt} incorporates greedy-inference consistency.
We share ConvexTok's vocabulary-selection relaxation. Our contribution is a price
certificate that evaluates it without a general-purpose solver on unrestricted document
graphs. Its optimisation and numerical verification are separate computations.
The decomposition separates shortest-path segmentation from vocabulary-budget selection.

\paragraph{Boundary relaxation and prediction.}
SuperBPE and BoundlessBPE construct vocabularies that span conventional
boundaries~\citep{superbpe,boundlessbpe}. Our tax interval instead bounds the compression
cost of the boundary rule independently of a vocabulary-construction heuristic.
PathPiece~\citep{pathpiece} and TokEval~\citep{tokeval2026} show why intrinsic compression
must be assessed alongside language-model quality. We connect these questions by
certifying compression optima, comparing prediction under explicit decoder protocols,
and formulating a budget on crossing permissions.

\section{The occurrence-price certificate}
\label{sec:problem}\label{sec:theory}

\subsection{Vocabulary selection and boundary constraints}
Let $\Sigma$ be the alphabet of $256$ bytes and let a corpus contain $B>0$ bytes in documents $\mathcal D=\{b^{(n)}\}_{n=1}^{N_{\mathrm{doc}}}$.
The candidate set $\mathcal C$ consists of all distinct corpus substrings of lengths
$2$ through $L$, where $L\ge2$. A vocabulary contains all $256$ bytes and at most
$K\ge1$ additional entries. With $\mathrm{seg}(b;V)$ denoting the minimum number of
tokens from $V$ whose concatenation is $b$, define
\begin{equation}
\OPT(K)=\min_{S\subseteq\mathcal C,\,|S|\le K}
\sum_n\mathrm{seg}(b^{(n)};\Sigma\cup S).
\label{eq:opt}
\end{equation}
Every document defines a candidate directed acyclic graph $G_n$ on byte positions,
containing byte edges and all candidate occurrences legal under the chosen cut regime.
A vocabulary $S$ makes usable the byte edges and longer edges whose strings belong to $S$.
Minimum-count encoding is a shortest path in this usable subgraph. Other inference rules can produce more tokens,
so a lower bound on \eqref{eq:opt} also bounds their achievable token counts.

A pre-tokeniser $\pi$ deletes occurrences that cross its cuts, preserving byte edges
and document boundaries. In a cut regime, occurrence sets contain only surviving edges;
candidate types with no surviving occurrence are omitted. Hence $\OPT_\pi(K)\ge\OPT_\varnothing(K)$.
We call
\begin{equation}
\tau_\pi(K)=\frac{\OPT_\pi(K)}{\OPT_\varnothing(K)}-1
\label{eq:tax}
\end{equation}
the \emph{optimal pre-tokenisation tax}. This compares two optima; the ratio of two
trained vocabularies is a different measurement of achieved compression.

\subsection{Prices and the certified bound}
The difficulty is that selecting one vocabulary entry enables all of its occurrences,
possibly in many documents. Prices temporarily separate these decisions: each path pays
for the occurrences it uses, while vocabulary selection collects their aggregate bids.
Subtracting the largest admissible collection of bids makes the separation conservative.

Write $x_e$ for edge use and $y_t$ for vocabulary inclusion. For edge $e$, let $t(e)$ denote its byte string and $E(t)$ the occurrences of string $t$.
With $A_n$ the document incidence matrix and $d_n$ a unit source-to-sink demand,
$A_nx^{(n)}=d_n$ enforces a unit flow. The equivalent integer programme is
\begin{equation}
\min_{x,y\in\{0,1\}}\sum_e x_e
\quad\text{s.t.}\quad
A_nx^{(n)}=d_n,\qquad x_e\le y_{t(e)},\qquad\sum_ty_t\le K.
\label{eq:ip}
\end{equation}
Vocabulary variables index multibyte strings only, and the linking constraint applies only to
multibyte edges; byte edges are always available. Relax it using prices
$\lambda_e\ge0$, and define the bid for token $t$ by
$\Lambda_t=\sum_{e\in E(t)}\lambda_e$. Write $\SP_n(1+\lambda)$ for the shortest-path cost in the complete candidate graph $G_n$ with cost
$1+\lambda_e$ on multibyte edges and cost $1$ on byte edges.

\begin{theorem}[Price certificate]
\label{thm:cert}
For any non-negative occurrence prices,
\begin{equation}
\boxed{\quad\OPT(K)\ \ge\ C_K(\lambda)
:=\sum_n\SP_n(1+\lambda)-\sum_{i=1}^{K}\Lambda_{(i)}.\quad}
\label{eq:cert}
\end{equation}
The bids are sorted in decreasing order and padded with zeros as necessary.
\end{theorem}

To see the bound directly, take any feasible vocabulary and segmentation.
The priced shortest path costs no more than this segmentation's token count plus the
prices of its used occurrences. Those prices sum to at most the bids of the vocabulary's
$K$ entries, and therefore to at most the top-$K$ bid sum. Subtraction gives
\eqref{eq:cert}. The proof does not require optimised prices.

\begin{theorem}[LP tightness of the price family]
\label{thm:nogap}
The maximum of $C_K(\lambda)$ over all $\lambda\ge0$ equals the optimum of
\eqref{eq:ip} relaxed to $x\ge0$, $y\in[0,1]$.
\end{theorem}

Thus the full price family loses nothing relative to the LP. A finite optimiser can
stop short of that maximum, and the LP can lie below the integer optimum.
Appendix~\ref{app:price-duality} gives the explicit LP dual and sufficient conditions for
integer equality.

\paragraph{Scalable price optimisation.}
\label{sec:algo}
Setting $\lambda_e=\theta/n_{t(e)}$, where $n_t=|E(t)|$, gives an inexpensive initialisation:
\begin{corollary}[Uniform bids]
\label{cor:unif}
For $\theta\ge0$, $\OPT(K)\ge\sum_n\SP_n(1+\theta/n_t)-K\theta$.
\end{corollary}
We optimise individual occurrence prices for a selected active set and assign every
remaining candidate a bid no larger than the $K$-th largest bid in the active set. Those candidates remain
in the shortest-path graph. This restricts the price family while retaining a bound for
the full vocabulary class (Proposition~\ref{prop:bg}). A suffix-array index supports
$O(BL)$ evaluation. Our system, \slot{} (Algorithm~\ref{alg:main}), combines projected
price ascent with vocabulary additions,
swaps and ruin-and-recreate proposals; complete vocabularies are rescored before acceptance.
Only accepted improvements change the upper bound (Appendix~\ref{app:search});
Appendix~\ref{app:ablations} separates the effects of the optimisation components.

\subsection{From a formula to conservative reported numbers}
\label{sec:numeric}
An independent checker converts saved prices into non-negative multiples of $2^{-32}$,
recomputes token bids exactly, and evaluates every shortest path with checked integer
arithmetic. Background prices are integer quotients whose total bid is no larger than
the active threshold. The resulting rational number is itself a valid dual value;
its validity does not depend on a floating-point error estimate or optimiser convergence.
The checker validates candidate identities and occurrence counts, rejects overflow, and
records input hashes (Appendices~\ref{app:price-verification} and~\ref{app:validation}).

Combining a positive integer lower bound $L_j$ with a feasible token count $U_j$ in each regime
$j\in\{\pi,\varnothing\}$ yields
\begin{equation}
\max\!\left\{0,\frac{L_\pi}{U_\varnothing}-1\right\}
\ \le\ \tau_\pi(K)\ \le\ \frac{U_\pi}{L_\varnothing}-1.
\label{eq:taxinterval}
\end{equation}
Reported endpoints are rounded outwards. These are deterministic optimisation intervals,
not sampling confidence intervals (Appendix~\ref{app:tax-budgets}).

\section{The certified cost of pre-tokenisation}
\label{sec:tax}
\input{results}

\section{Prediction at matched training budgets}
\label{sec:lm}
\input{lm}

\section{Boundary licences}
\label{sec:frontier}
\input{boundary_frontier}

\section{Discussion}
The certificate and the language-model comparisons answer complementary questions.
The first establishes how much compression a boundary rule excludes at fixed corpus,
length cap and vocabulary budget. The second compares prediction from the resulting
fitted dictionaries. A small optimisation gap inside a constrained problem cannot settle
either the cost of that constraint or its value for prediction. This distinction explains
why near-optimal constrained BPE and substantial gains from boundary relaxation can
coexist.

Boundary licences expose a further design choice: permissions can be allocated to selected
string types while legality remains occurrence-specific. The measured frontier places
much of the achieved compression gain at small licence budgets, making intermediate
policies concrete candidates for future language-model evaluation. Compression certificates
quantify the available sequence-length reduction; prediction measurements determine which reductions are useful.

\section{Conclusion}
Occurrence pricing certifies the compression cost of pre-tokenisation without requiring an
exact vocabulary search. The resulting intervals expose a substantial boundary cost that
constrained optimality comparisons conceal. The multilingual comparisons favour regex-constrained fitting in most languages,
including under independent tuning and evaluation, showing why
sequence length alone is insufficient to choose prediction units. Boundary licences extend
the certificate to intermediate policies and quantify how much compression they recover.

\section*{Limitations}
The certificates apply to specified finite corpora, vocabulary budgets and maximum token
lengths. Certification on other text domains requires their corresponding corpora.
The cut-placement control preserves character-length multisets while allowing byte-length variation.
Language-model results concern the measured architectures, training
budgets and decoder protocols, with three seeds per language limiting precision within each
language. Token matching fixes training-token allocation while byte exposure and context span
vary jointly; the measured contrast combines these changes with vocabulary content.
Vocabulary-search effort also differs between regimes; matching this effort would further
separate the fitting constraint from the quality of its optimisation. The largest
reported model has $1.88$B non-embedding parameters. Extending the prediction analysis to
larger models and to intermediate licence budgets remains future work.

\label{sec:mainend}
\section*{Reproducibility statement}
\label{sec:repro}
The supplementary package provides per-seed test measurements, compression endpoints and
a self-contained script that reconstructs their reported summary statistics.
Appendices~\ref{app:proofs},
\ref{app:setup}, \ref{app:lmfull} and~\ref{app:boundary-protocol} give proofs and protocols.
Appendix~\ref{app:checklist} gives the reconstruction commands and
Appendix~\ref{app:gpuplan} documents the independent tuning/evaluation replication.

\section*{AI statement}
Generative AI assisted code development, manuscript writing, mathematical and reference checks,
the authors reviewed and verified all AI-assisted work.
\bibliographystyle{iclr2027_conference}
\bibliography{refs}
\clearpage
\appendix
\raggedbottom

\section{Proofs and numerical certification}
\label{app:proofs}
\input{proofs}
\input{boundary_proofs}

\section{Compression protocol and complete results}
\label{app:setup}\label{app:extra}
\input{setup_full}
\input{compression_supplement}

\FloatBarrier
\input{lm_full}
\FloatBarrier
\input{boundary_supplement}

\FloatBarrier

\section{Related work and scope}
\label{app:related}
\input{related_full}

\section{Reproduction and independent replication}
\label{app:checklist}
\input{checklist}
\input{gpu_pending}

\end{document}

%% file: rebuild_numbers.tex
\providecommand{\dattaxEnLo}{}
\renewcommand{\dattaxEnLo}{28.3}
\providecommand{\dattaxEnHi}{}
\renewcommand{\dattaxEnHi}{36.8}

\providecommand{\datbpeEnRegex}{}
\renewcommand{\datbpeEnRegex}{4.551}
\providecommand{\datbpeGapEnRegex}{}
\renewcommand{\datbpeGapEnRegex}{2.1}

\providecommand{\datbpeEnNone}{}
\renewcommand{\datbpeEnNone}{5.697}
\providecommand{\datbpeGapEnNone}{}
\renewcommand{\datbpeGapEnNone}{10.9}

\providecommand{\datslotVsBpeEnNone}{}
\renewcommand{\datslotVsBpeEnNone}{4.39}
\providecommand{\dattaxZhLo}{}
\renewcommand{\dattaxZhLo}{9.8}
\providecommand{\dattaxZhHi}{}
\renewcommand{\dattaxZhHi}{11.0}

\providecommand{\dattaxMin}{}
\renewcommand{\dattaxMin}{10.1}

\providecommand{\dattaxMax}{}
\renewcommand{\dattaxMax}{56.2}

\providecommand{\datTestN}{}
\renewcommand{\datTestN}{12}
\providecommand{\datTestMean}{}
\renewcommand{\datTestMean}{9.25}
\providecommand{\datTestLo}{}
\renewcommand{\datTestLo}{0.93}
\providecommand{\datTestHi}{}
\renewcommand{\datTestHi}{17.31}
\providecommand{\datTestCiLo}{}
\renewcommand{\datTestCiLo}{6.54}
\providecommand{\datTestCiHi}{}
\renewcommand{\datTestCiHi}{11.92}
\providecommand{\datTestHolmN}{}
\renewcommand{\datTestHolmN}{2}
\providecommand{\datTestSignP}{}
\renewcommand{\datTestSignP}{0.00024}
\providecommand{\datTestSd}{}
\renewcommand{\datTestSd}{4.96}
\providecommand{\datTestSeedN}{}
\renewcommand{\datTestSeedN}{3}
\providecommand{\datTestPairN}{}
\renewcommand{\datTestPairN}{36}
\providecommand{\dattaxKmaxEnLo}{}
\renewcommand{\dattaxKmaxEnLo}{53.6}
\providecommand{\dattaxKmaxEnHi}{}
\renewcommand{\dattaxKmaxEnHi}{61.3}
\providecommand{\dattaxKmaxEnAch}{}
\renewcommand{\dattaxKmaxEnAch}{54.5}
\providecommand{\dattaxKmaxZhLo}{}
\renewcommand{\dattaxKmaxZhLo}{14.4}
\providecommand{\dattaxKmaxZhHi}{}
\renewcommand{\dattaxKmaxZhHi}{15.7}
\providecommand{\dattaxKmaxZhAch}{}
\renewcommand{\dattaxKmaxZhAch}{14.7}
\providecommand{\datbpeGapKSFirst}{}
\renewcommand{\datbpeGapKSFirst}{9.8}
\providecommand{\datbpeGapKSLast}{}
\renewcommand{\datbpeGapKSLast}{1.3}
\providecommand{\dattaxRegexEnvLo}{}
\renewcommand{\dattaxRegexEnvLo}{25.3}
\providecommand{\dattaxRegexEnvHi}{}
\renewcommand{\dattaxRegexEnvHi}{37.0}
\providecommand{\dattaxPenaltyRho}{}
\renewcommand{\dattaxPenaltyRho}{+0.22}
\providecommand{\dattaxPenaltyRhoP}{}
\renewcommand{\dattaxPenaltyRhoP}{0.49}

%% file: validation_numbers.tex
\providecommand{\datvalOursLo}{}
\renewcommand{\datvalOursLo}{99.59}
\providecommand{\datvalOursHi}{}
\renewcommand{\datvalOursHi}{99.99}

\providecommand{\datnValid}{}
\renewcommand{\datnValid}{14}

%% file: supplementary_numbers.tex
\providecommand{\datbbBudgetSlot}{}
\renewcommand{\datbbBudgetSlot}{+0.9}
\providecommand{\datbbBudgetSlotpi}{}
\renewcommand{\datbbBudgetSlotpi}{-0.1}
\providecommand{\datbbDefBytes}{}
\renewcommand{\datbbDefBytes}{-4.5}
\providecommand{\datbbDefTok}{}
\renewcommand{\datbbDefTok}{-3.6}
\providecommand{\datbbSeedN}{}
\renewcommand{\datbbSeedN}{2}
\providecommand{\datcoCorpusGB}{}
\renewcommand{\datcoCorpusGB}{52}
\providecommand{\datcoFirst}{}
\renewcommand{\datcoFirst}{+18.32}
\providecommand{\datcoFirstN}{}
\renewcommand{\datcoFirstN}{24}
\providecommand{\datcoLast}{}
\renewcommand{\datcoLast}{+2.65}
\providecommand{\datcoLastN}{}
\renewcommand{\datcoLastN}{387}
\providecommand{\datdepLongBenefitMax}{}
\renewcommand{\datdepLongBenefitMax}{0.022}
\providecommand{\datdsABlimpSlotVsSlotpi}{}
\renewcommand{\datdsABlimpSlotVsSlotpi}{-4.0}
\providecommand{\datdsAttenDiff}{}
\renewcommand{\datdsAttenDiff}{-1.1}
\providecommand{\datdsAttenDiffSd}{}
\renewcommand{\datdsAttenDiffSd}{0.5}
\providecommand{\datdsBBlimpSlotVsSlotpi}{}
\renewcommand{\datdsBBlimpSlotVsSlotpi}{-2.9}
\providecommand{\datdsDoDBootHi}{}
\renewcommand{\datdsDoDBootHi}{-0.8}
\providecommand{\datdsDoDBootLo}{}
\renewcommand{\datdsDoDBootLo}{-1.4}
\providecommand{\datdsUidAMorph}{}
\renewcommand{\datdsUidAMorph}{-6.9}
\providecommand{\datdsUidBMorph}{}
\renewcommand{\datdsUidBMorph}{-3.1}
\providecommand{\datlmEnSdRatio}{}
\renewcommand{\datlmEnSdRatio}{43}
\providecommand{\datnLmRuns}{}
\renewcommand{\datnLmRuns}{1186}
\providecommand{\datpanelSeeds}{}
\renewcommand{\datpanelSeeds}{1234, 5678, 9012}
\providecommand{\datpanelSelEdgeList}{}
\renewcommand{\datpanelSelEdgeList}{English (seeds 5678 and 9012, both arms); Turkish (seed 5678, regex arm); Finnish (seed 9012, unrestricted arm); Korean (seed 9012, unrestricted arm); Russian (seeds 5678 and 9012, regex arm)}
\providecommand{\datpanelSelEdgeN}{}
\renewcommand{\datpanelSelEdgeN}{9}
\providecommand{\datpanelSelInteriorN}{}
\renewcommand{\datpanelSelInteriorN}{63}
\providecommand{\datpanelSelN}{}
\renewcommand{\datpanelSelN}{72}
\providecommand{\datxxlBpeBpb}{}
\renewcommand{\datxxlBpeBpb}{0.78366}
\providecommand{\datxxlBpeXtwoMean}{}
\renewcommand{\datxxlBpeXtwoMean}{1.59178}
\providecommand{\datxxlBpeXtwoN}{}
\renewcommand{\datxxlBpeXtwoN}{2}
\providecommand{\datxxlBpeXtwoSd}{}
\renewcommand{\datxxlBpeXtwoSd}{0.00596}
\providecommand{\datxxlCorpusGB}{}
\renewcommand{\datxxlCorpusGB}{176}
\providecommand{\datxxlSlotVsBpe}{}
\renewcommand{\datxxlSlotVsBpe}{4.58}
\providecommand{\datxxlSlotpiBpb}{}
\renewcommand{\datxxlSlotpiBpb}{0.78117}
\providecommand{\datxxlSlotpiXtwoMean}{}
\renewcommand{\datxxlSlotpiXtwoMean}{0.78162}
\providecommand{\datxxlSlotpiXtwoN}{}
\renewcommand{\datxxlSlotpiXtwoN}{3}
\providecommand{\datxxlSlotpiXtwoSd}{}
\renewcommand{\datxxlSlotpiXtwoSd}{0.00040}
\providecommand{\datxxlSuperOffVsBpe}{}
\renewcommand{\datxxlSuperOffVsBpe}{0.72}
\providecommand{\datxxlSuperVsBpe}{}
\renewcommand{\datxxlSuperVsBpe}{2.87}

\providecommand{\datyyAnchorBpe}{}
\renewcommand{\datyyAnchorBpe}{-3.8}
\providecommand{\datyyAnchorSlotpi}{}
\renewcommand{\datyyAnchorSlotpi}{-2.9}
\providecommand{\datyyBandBpe}{}
\renewcommand{\datyyBandBpe}{0.5}
\providecommand{\datyyBandSlotpi}{}
\renewcommand{\datyyBandSlotpi}{0.5}
\providecommand{\datyyDefBpe}{}
\renewcommand{\datyyDefBpe}{-3.3}
\providecommand{\datyyDefSlotpi}{}
\renewcommand{\datyyDefSlotpi}{-3.3}

\providecommand{\datyyNBpe}{}
\renewcommand{\datyyNBpe}{2}

%% file: replication_numbers.tex
\providecommand{\datRepLangN}{}
\renewcommand{\datRepLangN}{12}
\providecommand{\datRepSeedN}{}
\renewcommand{\datRepSeedN}{3}
\providecommand{\datRepTuneSeedN}{}
\renewcommand{\datRepTuneSeedN}{2}

\providecommand{\datRepTokenBudget}{}
\renewcommand{\datRepTokenBudget}{629{,}145{,}600}
\providecommand{\datRepTokenMean}{}
\renewcommand{\datRepTokenMean}{6.99}

\providecommand{\datRepTokenPositiveN}{}
\renewcommand{\datRepTokenPositiveN}{11}

\providecommand{\datRepByteMean}{}
\renewcommand{\datRepByteMean}{6.70}

\providecommand{\datRepBytePositiveN}{}
\renewcommand{\datRepBytePositiveN}{11}

%% file: boundary_numbers.tex
\newcommand{\boundaryVocabBudget}{32{,}768}
\newcommand{\boundaryEnTenRecovery}{85.2}
\newcommand{\boundaryEnBptZero}{4.75}
\newcommand{\boundaryEnBptFull}{6.24}
\newcommand{\boundaryEnTenCiLo}{50.5}
\newcommand{\boundaryEnTenCiHi}{100.0}
\newcommand{\boundaryEnFitBytes}{8,434,722}
\newcommand{\boundaryZhTenRecovery}{100.0}
\newcommand{\boundaryZhBptZero}{4.84}
\newcommand{\boundaryZhBptFull}{5.22}
\newcommand{\boundaryZhTenCiLo}{38.7}
\newcommand{\boundaryZhTenCiHi}{100.0}
\newcommand{\boundaryZhFitBytes}{8,415,699}
\newcommand{\boundaryZhSatQ}{8{,}192}
\newcommand{\boundaryZhSatW}{3{,}276}
\newcommand{\boundaryZhSatBpt}{4.828}

%% file: results.tex
\paragraph{Protocol.}
How much compression does the boundary rule exclude, independently of the vocabulary learner?
The main compression study uses the 20231101 Wikipedia snapshot in twelve languages,
with approximately $42$\,MB per language, $K=32768$ non-byte entries and $L=16$ bytes.
The o200k-style regex~\citep{o200ktiktoken} and unrestricted regime preserve the same document boundaries.
BPE supplies the standard merge baseline, Unigram supplies a likelihood-based objective,
and Two-stage BPE tests boundary relaxation during merge training. Each baseline receives the same corpus, length cap and vocabulary budget. All reported
compression scores use minimum-count decoding of the released byte vocabulary under the
specified cuts. The Two-stage BPE baseline is our merge curriculum; the official released
SuperBPE vocabulary is evaluated separately in the large-vocabulary LM study.
Appendices~\ref{app:data} and~\ref{app:baselines} specify corpus preparation and baseline construction.

\input{tables/tab_headline}

\paragraph{A large cost hidden by a small constrained gap.}
English BPE lies $\datbpeGapEnRegex\%$ above the regex-constrained token-count lower bound.
Removing the cuts raises this gap to $\datbpeGapEnNone\%$
(Table~\ref{tab:headline}). The boundary constraint is therefore essential to the
near-optimality comparison: proximity to the constrained bound says little about the
unrestricted optimum. \slot{} improves on unrestricted BPE by
$\datslotVsBpeEnNone\%$ in achieved token count.

The English optimal tax lies in $[\dattaxEnLo,\dattaxEnHi]\%$;
Chinese yields $[\dattaxZhLo,\dattaxZhHi]\%$.
These intervals use a lower bound from one regime and a feasible vocabulary from the
other, as in \eqref{eq:taxinterval}. They quantify the cost of the boundary constraint
without treating an achieved vocabulary ratio as the optimum.

\begin{figure}[t]
\centering
\includegraphics[width=\textwidth]{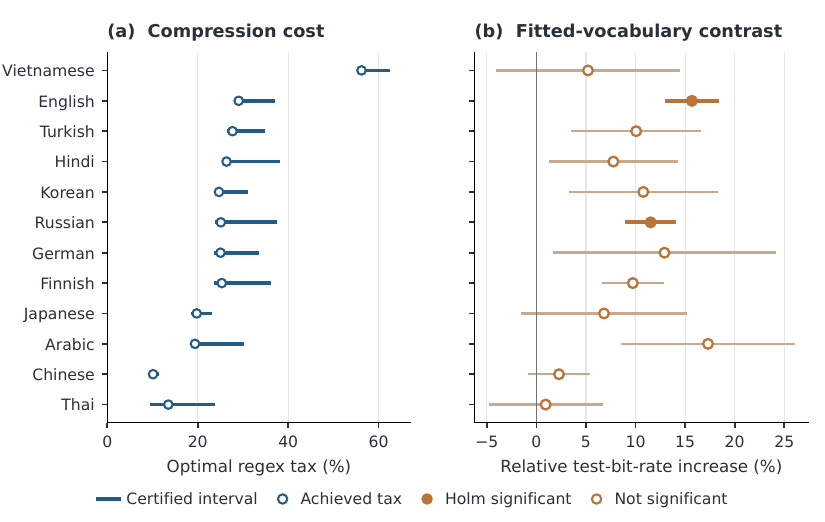}
\caption{\textbf{Compression cost and test prediction cost across languages.}
(a) Certified intervals for the optimal regex tax on Wikipedia; open circles show the
corresponding achieved vocabulary ratios.
(b) Paired test-bit-rate difference between unrestricted and regex-trained \slot{}
dictionaries after within-seed rate selection, with unadjusted $95\%$ paired-$t$ intervals over three seeds per language;
filled markers are Holm-significant at $0.05$.
Blue marks the compression measurements of panel (a); gold marks the prediction contrast
of panel (b).
The panels use their respective corpora and protocols; rows align the language-specific
measurements from the two studies. Languages share one ordering, sorted by the certified
lower endpoint of the panel-(a) interval. The two costs rank the languages differently:
Vietnamese pairs the highest certified tax with a prediction penalty below the panel
median, while Arabic pairs the largest prediction penalty with a certified tax below the
panel median (Tables~\ref{tab:compression_all} and~\ref{tab:testpanel}).}
\label{fig:tradeoff}
\end{figure}

\paragraph{Language variation.}
The achieved pre-tokenisation tax ranges from $\dattaxMin\%$ to $\dattaxMax\%$ across the panel.
Figure~\ref{fig:tradeoff} shows each language's optimisation interval. In the unrestricted regime,
\slot{} uses fewer tokens than both BPE and the local Two-stage baseline in all twelve languages.
Appendix~\ref{app:complete} gives the complete scores (Table~\ref{tab:compression_all}).

\paragraph{Budgets, regexes and deployed tokenisers.}
We next vary the vocabulary budget and boundary rule to test the scope of the headline result. In the vocabulary-budget sweep, the achieved tax and both endpoints of its certified
interval increase across the tested budgets in English and Chinese, reaching
$[\dattaxKmaxEnLo,\dattaxKmaxEnHi]\%$ in English (achieved $\dattaxKmaxEnAch\%$) and
$[\dattaxKmaxZhLo,\dattaxKmaxZhHi]\%$ in Chinese (achieved $\dattaxKmaxZhAch\%$) at the
largest budget, while the English constrained BPE gap narrows from $\datbpeGapKSFirst\%$ to
$\datbpeGapKSLast\%$ (Table~\ref{tab:ksweep}). Replacing the boundary regex by the
other deployed pre-tokenisers of Table~\ref{tab:regexes} leaves overlapping tax intervals
within $[\dattaxRegexEnvLo,\dattaxRegexEnvHi]\%$ on this corpus.
Appendix~\ref{app:sweeps} also reports the achieved compression of released vocabularies
under their native and minimum-count decoders (Table~\ref{tab:deployed}).

\paragraph{What limits the interval?}
\label{sec:validate}
Small-instance comparisons distinguish certificate optimisation from relaxation error.
On $\datnValid$ small instances, occurrence-price optimisation reaches
$\datvalOursLo$--$\datvalOursHi\%$ of the LP objective. Available integer solves
separate the LP--IP gap from vocabulary-search error (Table~\ref{tab:validate}; Appendix~\ref{app:small}).
A large primal--dual gap can reflect incomplete price optimisation, a suboptimal
vocabulary, or a gap between the LP and integer optima. The interval remains valid;
its width measures the compression uncertainty left by these three sources.

%% file: tables/tab_headline.tex
\begin{table}[t]
\centering
\small
\renewcommand{\arraystretch}{1.12}
\setlength{\tabcolsep}{4pt}
\caption{\textbf{Removing cuts changes the compression optimum.} Bytes per token (higher is better), with token-count gap to the integer-checked lower bound in parentheses (\%), rounded upwards. The italic row is a conservative upper ceiling on achievable bytes per token, rounded upwards. Two-stage BPE uses an unconstrained merge curriculum and has no regex-constrained entry (--). Unigram optimizes likelihood, a different objective from minimum token count. English and Chinese Wikipedia, $K=32768$, $L=16$, approximately $42$\,MB per language; all methods use minimum-count encoding under the indicated cuts.}
\label{tab:headline}
\begin{tabular}{lrrrr}
\toprule
& \multicolumn{2}{c}{English} & \multicolumn{2}{c}{Chinese}\\ \cmidrule(lr){2-3}\cmidrule(lr){4-5} Method & Regex & None & Regex & None\\
\midrule
\emph{Optimum: upper ceiling} & \emph{4.642} & \emph{6.314} & \emph{4.593} & \emph{5.084}\\
\slot{} & \textbf{4.616} (0.6) & \textbf{5.958} (6.0) & \textbf{4.582} (0.3) & \textbf{5.045} (0.8)\\
BPE & 4.551 (2.1) & 5.697 (10.9) & 4.545 (1.1) & 4.978 (2.2)\\
Two-stage BPE & -- & 5.898 (7.1) & -- & 4.967 (2.4)\\
Unigram & 2.894 (60.4) & 5.735 (10.1) & 4.006 (14.7) & 4.787 (6.3)\\
\bottomrule
\end{tabular}
\end{table}

%% file: lm.tex
The compression certificate measures sequence length. We next ask whether the additional vocabulary
freedom improves prediction. The primary comparison changes vocabulary construction while holding the
LM-stream decoder fixed: \slot{}$_\pi$ is fitted with regex boundaries and \slot{} without them,
then both encode complete documents by unrestricted minimum-count Viterbi decoding.
The contrast measures prediction from the resulting dictionaries under a common decoder.
Appendix~\ref{app:panelmech} records decoder protocols
and differences in vocabulary-search effort.

\paragraph{A held-out test comparison in twelve languages.}
For each language, the two arms share an $85$M non-embedding-parameter Transformer, a padded vocabulary
of $33{,}152$ entries, a $2048$-token context, an optimiser, a schedule, and $0.629$\,G training tokens.
Rates are selected separately for each arm and seed on a balanced validation grid. We evaluate the
selected runs on held-out test text and pair them over three training seeds. Bits per byte uses the
exact byte lengths of the scored target tokens (Appendix~\ref{app:testprotocol}).

The unrestricted fitted dictionaries have higher mean test bits per byte in all $\datTestN$ languages
(Appendix~\ref{app:lmfull}, Table~\ref{tab:testpanel}). The relative increases range from $\datTestLo\%$ to $\datTestHi\%$,
with an equal-language mean of $\datTestMean\%$ ($95\%$ language-bootstrap interval, conditional on observed language means,
$\datTestCiLo$--$\datTestCiHi\%$). A one-sided sign test gives $p=\datTestSignP$ under independent
language-level signs. Across the separate Wikipedia and language-model corpora, the achieved
tax and test penalty have Spearman rank correlation $\dattaxPenaltyRho$ (permutation
$p=\dattaxPenaltyRhoP$), a descriptive comparison across the twelve languages.

Effects vary across languages and seeds; Holm-corrected paired tests resolve $\datTestHolmN$
of the $\datTestN$ contrasts. Here, rates are selected within each evaluation seed.
The following comparison uses disjoint tuning and evaluation seeds (Appendix~\ref{app:gpuplan}).

\paragraph{Independent tuning and evaluation.}
Does the contrast persist when learning-rate selection uses different seeds from evaluation?
We select rates using $\datRepTuneSeedN$ tuning seeds and compare the selected configurations
on $\datRepSeedN$ independent evaluation seeds per language, retaining the same model
geometry and $\datRepTokenBudget$ training-token budget. Both dictionaries use unrestricted
Viterbi decoding. Training windows share sampled raw-byte start positions; evaluation
scores a common raw-byte interval under each decoder. Appendix~\ref{app:gpuplan} gives
the selection and context protocols.

Table~\ref{tab:replication} reports independent token blocks, byte-capped histories,
and sliding token contexts.
With $2048$-token blocks, unrestricted fitting increases mean test bits per byte in
$\datRepTokenPositiveN$ of $\datRepLangN$ languages, with an equal-language mean of
$\datRepTokenMean\%$. Capping the available history at $1024$ raw bytes retains the
same direction in $\datRepBytePositiveN$ languages and a mean of $\datRepByteMean\%$.
Thai has small negative mean contrasts in both modes, with Holm-adjusted $p$-values above
$0.05$. The positive cross-language mean persists under independent rate selection and
the common byte-history cap. History capping controls evaluation context;
training-byte exposure and vocabulary content differ between arms.

\input{tables/tab_replication_summary}

\paragraph{Scale and training allocation.}
Does the prediction penalty persist when model size and training budget increase together?
Across the measured English $D{=}20N$ ladder (20 training tokens per non-embedding parameter),
the \slot{} penalty against BPE falls from $\datcoFirst\%$ at $\datcoFirstN$M to
$\datcoLast\%$ at $\datcoLastN$M (Appendix~\ref{app:co}).
Figure~\ref{fig:scale} varies model size at a fixed token budget; Figure~\ref{fig:conv}
tracks learning against token and estimated byte exposure. Table~\ref{tab:lmscale}
summarises scale and training-allocation comparisons (Appendix~\ref{app:lrvar}).
At $996$M non-embedding parameters, $K=200$k and $D=20N$, locally trained Two-stage BPE
and the released SuperBPE vocabulary have penalties of $\datxxlSuperVsBpe\%$ and
$\datxxlSuperOffVsBpe\%$ against native BPE with regex boundaries.
These broader comparisons use different decoder protocols; the \slot{}/\slot{}$_\pi$
pair isolates dictionary differences under a common decoder.

\paragraph{Byte exposure and linguistic evaluation.}
Can training-byte exposure account for the downstream deficit?
In a separate $387$M control, the mean deficit on the Benchmark of Linguistic Minimal
Pairs (BLiMP; \citealp{warstadt2020blimp}) is larger under matched bytes than under matched
tokens ($\datbbDefBytes$ versus $\datbbDefTok$ points); the difference in budget response
remains unresolved with $\datbbSeedN$ seeds per cell (Appendix~\ref{app:downstream}).
Figure~\ref{fig:blimpuid} breaks down the $85$M and $996$M contrasts by linguistic field
and phenomenon. At $1.88$B, the English BLiMP comparisons use $\datyyNBpe$ checkpoints
per arm; their descriptive bands do not resolve whether attenuation continues.
Together, these controls establish the trade-off in the measured settings while leaving
its scaling and linguistic mechanism open.

%% file: tables/tab_replication_summary.tex
\begin{table}[t]
\centering
\small
\setlength{\tabcolsep}{5pt}
\renewcommand{\arraystretch}{1.07}
\caption{\textbf{Independent tuning preserves the average prediction penalty.} Relative test bits-per-byte change (\%) from the regex-fitted to the unrestricted dictionary; positive values favor the regex-fitted dictionary. Entries are paired means $\pm$ standard deviations over three evaluation seeds, disjoint from the two tuning seeds. $\dagger$ marks a two-sided paired test passing Holm correction over twelve languages within the indicated evaluation mode ($\alpha=0.05$).}
\label{tab:replication}
\begin{tabular}{lrrr}
\toprule
Language & Token blocks & Byte history & Sliding context\\
\midrule
English & $+10.44\,\pm\,1.69$ & $+9.90\,\pm\,1.19^{\dagger}$ & $+10.68\,\pm\,1.81$\\
German & $+11.06\,\pm\,0.84^{\dagger}$ & $+9.61\,\pm\,0.74^{\dagger}$ & $+11.36\,\pm\,0.85^{\dagger}$\\
Russian & $+7.45\,\pm\,1.77$ & $+5.75\,\pm\,1.62$ & $+7.94\,\pm\,1.83$\\
Finnish & $+6.41\,\pm\,1.03$ & $+6.01\,\pm\,0.70^{\dagger}$ & $+6.57\,\pm\,1.11$\\
Turkish & $+8.37\,\pm\,0.69^{\dagger}$ & $+6.88\,\pm\,0.59^{\dagger}$ & $+8.84\,\pm\,0.67^{\dagger}$\\
Vietnamese & $+3.42\,\pm\,0.58$ & $+3.78\,\pm\,0.48^{\dagger}$ & $+3.44\,\pm\,0.56$\\
Arabic & $+13.69\,\pm\,1.31^{\dagger}$ & $+11.13\,\pm\,1.46^{\dagger}$ & $+14.04\,\pm\,1.35^{\dagger}$\\
Hindi & $+10.22\,\pm\,0.38^{\dagger}$ & $+14.73\,\pm\,0.80^{\dagger}$ & $+10.58\,\pm\,0.22^{\dagger}$\\
Thai & $-0.59\,\pm\,0.36$ & $-0.11\,\pm\,0.34$ & $-0.39\,\pm\,0.35$\\
Korean & $+5.35\,\pm\,0.74$ & $+5.12\,\pm\,0.76^{\dagger}$ & $+5.43\,\pm\,0.71^{\dagger}$\\
Japanese & $+3.36\,\pm\,0.65$ & $+2.62\,\pm\,0.61$ & $+3.16\,\pm\,0.38^{\dagger}$\\
Chinese & $+4.62\,\pm\,1.07$ & $+4.95\,\pm\,1.27$ & $+5.36\,\pm\,0.96$\\
\bottomrule
\end{tabular}
\end{table}

%% file: boundary_frontier.tex
\paragraph{A nested family of boundary permissions.}
Boundary licences measure how compression changes as crossing permissions expand,
with the global vocabulary budget and candidate universe held fixed.

Let $V$ contain at most $K$ multibyte strings and let $W\subseteq V$ contain
at most $q$ \emph{crossing-licensed} strings. An occurrence wholly within a
pretoken may use any string in $V$; an occurrence crossing one or more cuts
must use a string in $W$. All byte tokens remain available without charge.
The licence belongs to a string, while legality is checked at each occurrence:
a type can appear both within a pretoken and across a cut. Thus $q$ counts
licensed types, irrespective of how often they cross.

Write $\mathrm{OPT}_{\pi,q}(K)$ for the minimum token count under this rule.
The feasible sets expand with $q$, with exact endpoints
\begin{equation}
\mathrm{OPT}_{\pi,0}(K)=\mathrm{OPT}_{\pi}(K),\qquad
\mathrm{OPT}_{\pi,K}(K)=\mathrm{OPT}_{\varnothing}(K).
\label{eq:boundary-endpoints}
\end{equation}
In this family, the occurrence rule is enforced when scoring fitting and held-out
compression. Section~\ref{sec:lm} uses a separate fitted-dictionary comparison
with unrestricted Viterbi decoding.

\begin{samepage}
\paragraph{A certificate for intermediate budgets.}
The certificate extends without assigning a string to a single boundary
class. For non-negative occurrence prices, let $A_t$ and $B_t$ be the total
prices on within-pretoken and crossing occurrences of string $t$,
respectively. Replace the top-$K$ term by
\begin{equation}
H_{K,q}(A,B)=
\max_{W\subseteq V,\,|V|\le K,\,|W|\le q}
\left(\sum_{t\in V}A_t+\sum_{t\in W}B_t\right).
\label{eq:boundary-support}
\end{equation}
Then $\sum_n\mathrm{SP}_n(1+\lambda)-H_{K,q}(A,B)$ is a lower bound on
$\mathrm{OPT}_{\pi,q}(K)$.\par
\end{samepage}

To evaluate the support exactly, pad with zero-bid types to obtain $m\ge K$ candidates
and sort them by decreasing $B_t$. In an optimal vocabulary, the $q$ licences can be
assigned to its first $q$ selected types in this order. A split $p$ separates them
from the $K-q$ unlicensed selections, giving
\begin{equation}
H_{K,q}(A,B)=\max_{q\le p\le m-K+q}
\left\{\operatorname{TopSum}_{q}(A+B;1{:}p)
+\operatorname{TopSum}_{K-q}(A;p+1{:}m)\right\}.
\label{eq:boundary-split}
\end{equation}
Here $\operatorname{TopSum}_r(v;I)$ sums the $r$ largest entries indexed by $I$, with
$\operatorname{TopSum}_0=0$. Prefix and suffix heap passes evaluate all splits in
$O(m\log(K+1))$ time after sorting. Appendix~\ref{app:boundary-proof} proves the formula
and relates the resulting certificate to the joint LP relaxation.

\begin{samepage}
\paragraph{Measured compression frontiers.}
How much of the endpoint gain requires only a small licence budget?
The English and Chinese compression experiments test $q=\lfloor fK\rfloor$ for
$f\in\{0,0.02,0.05,0.1,0.25,1\}$. Each point exports its vocabulary,
licences, full candidate graph and dyadic certificate witness. Figure~\ref{fig:boundary-frontier} distinguishes achieved endpoint recovery from
certified bounds on optimal recovery on the fitting corpus. Disjoint held-out records
provide compression and token diagnostics.\par
\end{samepage}
\input{boundary_results}

\paragraph{A cut-placement control.}
Does preserving piece lengths also preserve the compression cost of a boundary rule?
\input{boundary_control_results}

The full protocols, held-out compression and token diagnostics appear in
Appendix~\ref{app:boundary-protocol}.

%% file: boundary_results.tex
\begin{figure}[t]
\centering
\includegraphics[width=\textwidth]{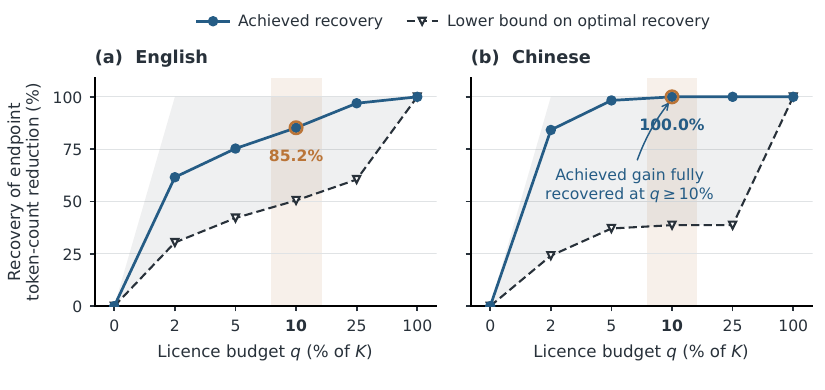}
\caption{\textbf{Small licence budgets recover most of the achieved endpoint gain.} Compression recovery as the boundary-licence budget grows, at $K=\boundaryVocabBudget$
multibyte slots and $L=16$; the fitting corpora contain \boundaryEnFitBytes{} English and \boundaryZhFitBytes{}
Chinese bytes. (a) English. (b) Chinese. Solid curves: achieved share of the endpoint
token-count reduction, comparing exported occurrence-aware vocabularies. Dashed curves
and shaded wedge: certified lower bound on the recovered share of the \emph{optimal}
endpoint gain (Equation~\ref{eq:boundary-recovery}); at every interior budget the
certified upper bound is $100\%$. Recovery is exactly $0\%$ at $q=0$ and $100\%$ at $q=K$ by definition. These are optimisation
bounds, not sampling uncertainty. The highlighted column marks the $10\%$ licence budget.
Tested budgets are equally spaced for display; connecting lines do not
interpolate certificates.}
\label{fig:boundary-frontier}
\end{figure}

For English, achieved fitting compression rises from $\boundaryEnBptZero$ to $\boundaryEnBptFull$ bytes per token across the endpoints. The $10\%$ licence budget recovers $\boundaryEnTenRecovery\%$ of this achieved token-count reduction. Its certified share of the optimal endpoint gain lies in $[\boundaryEnTenCiLo,\boundaryEnTenCiHi]\%$.
For Chinese, achieved fitting compression rises from $\boundaryZhBptZero$ to $\boundaryZhBptFull$ bytes per token across the endpoints. The $10\%$ licence budget recovers $\boundaryZhTenRecovery\%$ of this achieved token-count reduction. Its certified share of the optimal endpoint gain lies in $[\boundaryZhTenCiLo,\boundaryZhTenCiHi]\%$. Chinese licence use saturates: offered $q=\boundaryZhSatQ$, the fit takes only $\boundaryZhSatW$ licences, with held-out compression unchanged from the $10\%$ rung ($\boundaryZhSatBpt$ bytes per token).
At a $2\%$ licence budget, the achieved count is below the zero-licence certified lower bound in both languages, proving an improvement over every zero-licence vocabulary in this model on these fitting corpora.
Figure~\ref{fig:boundary-frontier} plots both series; Tables~\ref{tab:boundary-frontier} and~\ref{tab:boundary-recovery} report the per-budget diagnostics and the certified recovery intervals.

%% file: boundary_control_results.tex
A separate $K=8{,}192$ cut-placement control preserves each document's cut count and Unicode-character piece-length multiset while permuting their order. At zero licences, the relative change in optimal token count from regex to random cuts is enclosed by independent certificates. For English, the optimal token-count change is $[47.1,53.5]\%$. For Chinese, the optimal token-count change is $[3.9,6.5]\%$. Thus cut placement affects optimal compression even when the cut count and Unicode-character piece-length multiset are fixed. Table~\ref{tab:boundary-controls} gives the complete budget grid. At full licences both cut sets define the same true optimum.

%% file: proofs.tex
Throughout, the corpus is nonempty, $\mathcal{C}$ and $G_n$ are defined in
\S\ref{sec:problem}, and $K\ge1$ is an integer. In a cut regime, $E(t)$ contains
only surviving occurrences. For the cardinality-constrained problem, types with no
surviving occurrence are omitted without changing the optimum, so $n_t=|E(t)|>0$.
Write
$\Lambda_t=\sum_{e\in E(t)}\lambda_e$ and order the bids as
$\Lambda_{(1)}\ge\Lambda_{(2)}\ge\cdots$, padding with zeros beyond $|\mathcal{C}|$.
Set $\lambda_e=0$ on byte edges. For a chosen shortest-path segmentation, $m_t$ counts the
used occurrences of $t$ and $\rho_t=m_t/n_t$.

\subsection{Vocabulary selection and price duality}
\label{app:price-duality}
\label{app:prop1}
\label{app:lpdual}

\paragraph{Integer formulation.}
\begin{proposition}
\label{prop:ipexact}
$\mathrm{val}(\mathrm{IP})=\OPT(K)$ for the program \eqref{eq:ip}.
\end{proposition}
\begin{proof}
($\le$) Let $S^\star$ attain \eqref{eq:opt}. Set $y=\mathbf{1}_{S^\star}$ and let $x^{(n)}$ be the
indicator of a shortest $0\!\to\!|b^{(n)}|$ path in $G_n$ restricted to edges usable under $S^\star$;
byte edges ensure that such a path exists. The pair
$(x,y)$ is feasible and has objective $\sum_n\mathrm{seg}(b^{(n)};\Sigma\cup S^\star)=\OPT(K)$.

($\ge$) Let $(x,y)$ be feasible and integral and put $S=\{t:y_t=1\}$, so $|S|\le K$. The support of
a $0/1$ unit flow is a single directed path plus circulations, and a DAG has no directed cycles, so
the flow is a single directed path; the linking constraints force every token
edge on that path to have $t(e)\in S$. Hence the path is a valid segmentation of $b^{(n)}$ over
$\Sigma\cup S$, its length is at least $\mathrm{seg}(b^{(n)};\Sigma\cup S)$, and summing over $n$
gives a value at least that of \eqref{eq:opt} at $S$, hence at least $\OPT(K)$.
\end{proof}

\paragraph{The LP and its dual.}

Let (LP) be \eqref{eq:ip} with $x\ge0$ and $y\in[0,1]$ (the constraint $x\le1$ is implied by unit
flow on a DAG and may be dropped). Assigning free potentials $\phi^n_i$ to the flow equalities,
$\lambda_e\ge0$ to $x_e\le y_{t(e)}$, $\theta\ge0$ to $\sum_ty_t\le K$ and $w_t\ge0$ to $y_t\le1$,
the dual of (LP) is
\begin{equation}
\label{eq:dual}
\begin{aligned}
\text{(D)}\qquad \max_{\lambda,\theta,w\,\ge\,0,\ \phi\text{ free}}\ \
 &\sum_n\left(\phi^n_{|b^{(n)}|}-\phi^n_0\right)\;-\;K\theta\;-\;\sum_{t\in\mathcal{C}}w_t\\[-2pt]
\text{s.t.}\qquad
 &\phi^n_j-\phi^n_i\,\le\,1+\lambda_e \quad\text{for every token edge } e=(i,j)\in G_n,\\[-2pt]
 &\phi^n_{i+1}-\phi^n_i\,\le\,1 \quad\text{for every byte edge},\\[-2pt]
 &\Lambda_t\,\le\,\theta+w_t \quad\text{for every } t\in\mathcal{C}.
\end{aligned}
\end{equation}
For an (LP)-feasible $(x,y)$ and a (D)-feasible $(\phi,\lambda,\theta,w)$, the
first two constraints of \eqref{eq:dual} read $\phi_{j(e)}-\phi_{i(e)}-\lambda_e\le1$ for every edge,
and $x\ge0$, so
\begin{equation}
\sum_e x_e\ \ge\ \sum_e\left(\phi_{j(e)}-\phi_{i(e)}\right)x_e-\sum_e\lambda_ex_e
\ =\ \sum_n\left(\phi^n_{|b^{(n)}|}-\phi^n_0\right)-\sum_e\lambda_ex_e ,
\label{eq:wd1}
\end{equation}
the last step by flow conservation: $x^{(n)}$ carries one unit from node $0$ to node $|b^{(n)}|$, so
the telescoping sum of potential differences is $\phi^n_{|b^{(n)}|}-\phi^n_0$. For the remaining term,
\begin{equation}
\sum_e\lambda_ex_e=\sum_{t\in\mathcal{C}}\sum_{e\in E(t)}\lambda_ex_e
\ \le\ \sum_t\Lambda_ty_t\ \le\ \sum_t(\theta+w_t)y_t\ \le\ K\theta+\sum_tw_t ,
\label{eq:wd2}
\end{equation}
using $x_e\le y_{t(e)}$ with $\lambda\ge0$, then $\Lambda_t\le\theta+w_t$ with $y\ge0$, then
$\sum_ty_t\le K$ with $\theta\ge0$ and $y_t\le1$ with $w\ge0$. Combining \eqref{eq:wd1} and
\eqref{eq:wd2} gives $\mathrm{val}(\mathrm{D})\le\mathrm{val}(\mathrm{LP})$.

\paragraph{Price certificate (Theorem~\ref{thm:cert}).}
\begin{proof}
Fix $\lambda\ge0$ and construct a (D)-feasible point explicitly. Let $\phi^n_i$ be the length of a
shortest $0\!\to\!i$ path in $G_n$ under costs $1$ on byte edges and $1+\lambda_e$ on token edges
(all token edges present), so $\phi^n_0=0$, $\phi^n_{|b^{(n)}|}=\SP_n(1+\lambda)$, and the two
potential constraints of \eqref{eq:dual} hold by definition of a shortest-path distance. Put
$\theta:=\Lambda_{(K)}$ and $w_t:=\max(0,\Lambda_t-\theta)$; then $\lambda,\theta,w\ge0$ and
$\Lambda_t\le\theta+w_t$ for every $t$, so the point is feasible. Its objective is
\[
\sum_n\SP_n(1+\lambda)-K\Lambda_{(K)}-\sum_{t}\max(0,\Lambda_t-\Lambda_{(K)})
=\sum_n\SP_n(1+\lambda)-\sum_{i=1}^{K}\Lambda_{(i)} ,
\]
because with $r:=|\{t:\Lambda_t>\Lambda_{(K)}\}|\le K-1$ we have
$K\Lambda_{(K)}+\sum_{i\le r}(\Lambda_{(i)}-\Lambda_{(K)})=\sum_{i\le r}\Lambda_{(i)}+(K-r)\Lambda_{(K)}
=\sum_{i\le K}\Lambda_{(i)}$, the last equality because ranks $r{+}1,\dots,K$ all carry the value
$\Lambda_{(K)}$. Weak duality for \eqref{eq:dual} and
$\mathrm{val}(\mathrm{LP})\le\mathrm{val}(\mathrm{IP})=\OPT(K)$ (Proposition~\ref{prop:ipexact}) give
\eqref{eq:cert}.
\end{proof}

For later use, eliminating the potentials and $w$ at fixed $(\lambda,\theta)$ gives
\begin{equation}
\label{eq:lag}
\mathcal L(\lambda,\theta)
=\sum_n\SP_n(1+\lambda)-K\theta
-\sum_{t\in\mathcal C}\max(0,\Lambda_t-\theta).
\end{equation}

\begin{remark}[Optimal budget multipliers]
\label{rem:thetaties}
At fixed $\lambda$, the maximisers of $\mathcal{L}(\lambda,\theta)$ over $\theta\ge0$ form
$[\Lambda_{(K+1)},\Lambda_{(K)}]$. On each open interval between bid values, the slope is
$-K+|\{t:\Lambda_t>\theta\}|$; it is positive below the maximising interval and negative
above it. Every point in the interval gives
$\sum_n\SP_n(1+\lambda)-\sum_{i\le K}\Lambda_{(i)}$, including the selected endpoint
$\theta=\Lambda_{(K)}$. Zero padding covers $K\ge|\mathcal{C}|$.
\end{remark}

\paragraph{LP tightness of the price family (Theorem~\ref{thm:nogap}).}
\begin{proof}
At fixed $(\lambda,\theta)$, maximising the dual objective \eqref{eq:dual} over the
potentials gives the shortest-path value $\sum_n\SP_n(1+\lambda)$.
The optimal upper-bound multipliers are $w_t=\max(0,\Lambda_t-\theta)$.
Consequently,
\[
\max_{\phi,w:\,\mathrm{(D)}\text{ feasible}}\mathrm{obj}(\phi,\lambda,\theta,w)
=\mathcal{L}(\lambda,\theta).
\]
The primal LP is feasible by the byte-only segmentation and has non-negative objective.
LP strong duality therefore gives
$\max_{\lambda,\theta\ge0}\mathcal{L}(\lambda,\theta)=\mathrm{val}(\mathrm{LP})$,
with the maximum attained. Maximising over $\theta$ as in
Remark~\ref{rem:thetaties} gives the equivalent occurrence-price formulation
$\max_{\lambda\ge0}C_K(\lambda)=\mathrm{val}(\mathrm{LP})$, where $C_K$ is the
right-hand side of \eqref{eq:cert}.
\end{proof}

\begin{remark}[Optimisation and integrality gaps]
\label{rem:whatnot}
Theorem~\ref{thm:nogap} characterises the maximum over all occurrence prices. Before integer
rounding, a finite run, a restricted price family, and a feasible vocabulary can each leave a gap:
\[
\mathrm{LB}\ \le\ \mathrm{val}(\mathrm{LP})\ \le\ \OPT(K)\ \le\ \mathrm{UB}.
\]
The first difference is dual optimisation error, the second is the integrality gap, and the
third is primal suboptimality. Table~\ref{tab:validate} compares these quantities using numerical HiGHS reference solves on
tractable instances. For a positive lower bound, the reported $\mathrm{UB}/\mathrm{LB}-1$ bounds a vocabulary's relative
suboptimality from above without assuming that any of these differences vanishes. The integer
rounding in Proposition~\ref{prop:integer} strengthens the bound on $\OPT(K)$ and is not used
when comparing a dual objective with the LP optimum.
\end{remark}

\paragraph{Integer equality.}

\begin{corollary}[Self-certifying exactness]
\label{cor:exact}
Assume $K\le|\mathcal C|$. Suppose a shortest-path segmentation at $\lambda\ge0$ and a set $S^\star$ of $K$
largest bids satisfy $m_t=n_t$ on $S^\star$ and $m_t=0$ outside it.
Then $S^\star$ is an optimal vocabulary and
\[
C_K(\lambda)=\mathrm{val}(\mathrm{LP})=\OPT(K)=\mathrm{UB},
\]
where $\mathrm{UB}$ is the token count of that segmentation.
\end{corollary}
\begin{proof}
The segmentation is feasible for $\Sigma\cup S^\star$, so
$\OPT(K)\le\mathrm{UB}$. It uses every occurrence of each selected candidate and no
other token edges; hence
\[
\sum_n\SP_n(1+\lambda)=\mathrm{UB}+\sum_{t\in S^\star}\Lambda_t
=\mathrm{UB}+\sum_{i\le K}\Lambda_{(i)}.
\]
Thus $C_K(\lambda)=\mathrm{UB}$, and
$C_K(\lambda)\le\mathrm{val}(\mathrm{LP})\le\OPT(K)\le\mathrm{UB}$ forces equality
throughout. Occurrence prices need not be constant within a candidate.
\end{proof}

\subsection{Restricted pricing and exact verification}
\label{app:price-verification}

\paragraph{Restricted price families.}
\begin{lemma}[Restriction]
\label{lem:restrict}
\eqref{eq:cert} holds for \emph{every} $\lambda\ge0$ with no further conditions. In particular it
holds when coordinates of $\lambda$ are tied to one another, fixed to constants, or optimised over
any subset --- restricting the price family can only weaken the bound, never invalidate it. Likewise
$\sum_{i\le K}\Lambda_{(i)}$ must be the sum of the $K$ largest bids \emph{of the multiset over
$\mathcal{C}$}: if several distinct candidates share a price they contribute one term each.
\end{lemma}

\paragraph{Uniform prices (Corollary~\ref{cor:unif}).}
\begin{proof}
Set $\lambda_e=\theta/n_{t(e)}$. Every candidate has bid $\theta$, so for
$K\le|\mathcal{C}|$ the top-$K$ sum is $K\theta$. Each path has cost
$|P|+\theta\sum_{e\in P\cap\mathrm{tok}}1/n_{t(e)}$, affine in $\theta$.
Its minimum over paths, minus $K\theta$, is therefore concave and piecewise linear.
When $K>|\mathcal{C}|$, subtracting $K\theta$ gives a weaker but still valid bound.

At a differentiability point, the derivative of this one-dimensional bound is
$\sum_t\rho_t-K$. An interior differentiable maximiser therefore satisfies
$\sum_t\rho_t=K$; at an interior kink, the left and right derivatives bracket zero.
A maximiser at $\theta=0$ instead requires only a nonpositive right derivative.
Thus usage fractions express the marginal response to a common slot price; they need not
sum to $K$ for an individual segmentation selected at a kink.
\end{proof}

\paragraph{Background pricing.}
\begin{proposition}[Background pricing]
\label{prop:bg}
Let $\mathcal A\subseteq\mathcal{C}$ with $|\mathcal A|\ge K$ be any active set, choose $\Lambda_t\ge0$ freely for
$t\in \mathcal A$, put $\theta:=$ the $K$-th largest bid over $\mathcal A$, and set $\Lambda_t:=\theta$ for every
$t\in\mathcal{C}\setminus \mathcal A$ (ties resolved in favour of $\mathcal A$). Then a set of $K$ largest bids of the
full multiset lies in $\mathcal A$, and \eqref{eq:cert} holds with the sum taken over $\mathcal A$ only.
\end{proposition}
\begin{proof}
Every $t\notin \mathcal A$ has $\Lambda_t=\theta$ and, by the choice of $\theta$ as the $K$-th largest bid
over $\mathcal A$, at least $K$ members of $\mathcal A$ have $\Lambda_t\ge\theta$. A set of $K$ largest bids of the
multiset over $\mathcal{C}$ may therefore be chosen inside $\mathcal A$, so
$\sum_{i\le K}\Lambda_{(i)}=\sum_{t\in\mathcal{T}}\Lambda_t$ for some $\mathcal{T}\subseteq \mathcal A$ with
$|\mathcal{T}|=K$, and Theorem~\ref{thm:cert} applies verbatim. Setting $\Lambda_t=\theta$ for
$t\notin \mathcal A$ is realised in the shortest-path pass by pricing each occurrence of such a $t$ at
$\theta/n_t$, which is legal for any $\lambda\ge0$ by Lemma~\ref{lem:restrict}.
\end{proof}

\begin{remark}[Choosing the active set]
\label{rem:Afree}
For any $\mathcal A$ with $|\mathcal A|\ge K$, the bound applies to vocabularies over the full candidate set
$\mathcal{C}$. Choosing $\mathcal A$ affects tightness and computational cost. We select the $M$
candidates with largest $n_t(|t|-1)$, an upper bound on their savings relative to byte-only
segmentation. Inactive hapaxes can share a stored representation because each distinct
candidate retains its own bid $\theta$ and remains present in the shortest-path problem.
\end{remark}

\paragraph{Exact verification.}
\begin{lemma}[Certificate checking]
\label{lem:check}
Given $(\lambda,K)$ and the corpus, the right-hand side of \eqref{eq:cert} is computable by one
shortest-path pass per document plus a top-$K$ selection, in $O(BL+|\mathcal{C}|)$ time, provided the corpus, edge indexing and prices are specified. Validity requires only
$\lambda\ge0$; checking does not require the optimisation trajectory.
\end{lemma}

\begin{proposition}[Exact numerical certificate]
\label{rem:float}
\label{prop:integer}
Fix an integer scale $Q>0$ and an active set $\mathcal A\subseteq\mathcal C$.
Give each active occurrence a non-negative integer price $a_e$, and form the exact bids
$\widehat{\Lambda}_t=\sum_{e\in E(t)}a_e$ for $t\in \mathcal A$. Let $T$ be their top-$K$ sum, padding with zeros.
Set $h=\widehat{\Lambda}_{(K)}$ when $K<|\mathcal A|$, and $h=0$ otherwise. Give each inactive occurrence the
integer price $a_e=\lfloor h/n_{t(e)}\rfloor$. If $D_{\rm int}$ is the sum of shortest-path costs
with byte cost $Q$ and token-edge cost $Q+a_e$, then
\[
\frac{D_{\rm int}-T}{Q}\ \le\ \mathrm{val}(\mathrm{LP})\ \le\ \OPT(K),
\qquad
\left\lceil\frac{D_{\rm int}-T}{Q}\right\rceil\ \le\ \OPT(K).
\]
\end{proposition}
\begin{proof}
The rational prices $\lambda_e=a_e/Q$ are non-negative. Each inactive candidate has total
integer bid $n_t\lfloor h/n_t\rfloor\le h$. When $K<|\mathcal A|$, at least $K$ active bids are
at least $h$; otherwise all inactive bids are zero. Thus $T$ is the top-$K$ sum over the
full candidate set, including every distinct hapax. Scaling all edge costs by $Q$ shows
that $(D_{\rm int}-T)/Q$ is exactly the certificate of Theorem~\ref{thm:cert}. Its LP bound follows
from the explicit dual construction. The final inequality uses that $\OPT(K)$ is an
integer.
\end{proof}

We use $Q=2^{32}$ and convert each saved active float32 price to
$a_e=\lfloor Q\lambda_e^{\mathrm{saved}}\rfloor$ directly from its IEEE-754 bits.
This defines a new dual point; validity does not depend on the accuracy of the
floating-point optimiser. An independent C++ checker uses integer arithmetic for
price conversion, bids, top-$K$ selection and the full shortest-path computation.
It checks accumulator ranges and rejects negative, nonfinite or overflowing inputs.
For a uniform certificate with integer bid ceiling $h$, the same construction uses
$a_e=\lfloor h/n_t\rfloor$ for every candidate and subtracts
$\min(K,|\mathcal C|)h$, an upper bound on its top-$K$ sum.

The checker reconstructs every legal span from the recorded corpus and document
boundaries. It verifies that the suffix index is a permutation ordered by length-$L$
prefixes, independently recomputes prefix matches from the bytes, and checks each
candidate's complete occurrence group and frequency. Distinct hapax strings may
share a storage class but retain separate bids. The artefact records hashes of the
corpus, boundaries, price file, active types, configuration and checker, together with
the exact numerator and denominator. Reported token-count lower bounds are rounded
up to integers; ratio endpoints are then rounded outwards by integer division.

\begin{remark}[A supergradient rate for the full price family]
\label{rem:rate}
Let $C_K(\lambda)=\sum_n\SP_n(1+\lambda)-\sum_{i\le K}\Lambda_{(i)}$.
All coordinates and edge sums in this remark range over multibyte occurrences.
For a shortest-path solution $x^\star$ and a top-$K$ indicator $s$, a supergradient is
$g_e=x_e^\star-s_{t(e)}$. Define the diagonal matrix $\mathsf P$ by $\mathsf P_{ee}=n_{t(e)}$.
Starting from $\lambda_0\ge0$ with step sizes $\eta_k>0$, momentum-free updates
$\lambda_{k+1}=[\lambda_k+\eta_k\mathsf P^{-1}g_k]_+$ are projected supergradient ascent in
coordinates $\mu=\mathsf P^{1/2}\lambda$, where
$\|\mathsf P^{-1/2}g_k\|^2\le\sum_e1/n_{t(e)}=|\mathcal{C}|$.
If $R=\|\mathsf P^{1/2}(\lambda_0-\lambda^\star)\|$ for a maximiser $\lambda^\star$, the standard
projection inequality gives
\[
C_K(\lambda^\star)-\max_{0\le k<T}C_K(\lambda_k)
\ \le\ \frac{R^2+|\mathcal{C}|\sum_{k<T}\eta_k^2}{2\sum_{k<T}\eta_k}.
\]
A constant step proportional to $T^{-1/2}$ gives an $O(T^{-1/2})$ bound; the schedule
$\eta_k\propto(k+1)^{-1/2}$ gives $O(\log T/\sqrt T)$. This result applies to the full,
independently parameterised price family. The production schedule also uses momentum and
background prices coupled to the active bids, so this rate does not apply to it.
Every evaluated price vector remains a valid certificate.
\end{remark}

\subsection{Certified cost intervals and general budgets}
\label{app:tax-budgets}
\label{app:budgets}

\paragraph{Tax intervals.}
\begin{proposition}
\label{prop:tax}
For every pre-tokeniser $\pi$ and budget $K$, $\OPT_\pi(K)\ge\OPT_\varnothing(K)$.
The tax $\tau_\pi(K)$ is non-negative and nondecreasing under refinement of the cut set.
Given non-negative lower bounds and valid upper bounds on both optima, with $\mathrm{LB}_\varnothing>0$,
\[
\tau_\pi(K)\ \in\
\left[\max\!\left\{0,\frac{\mathrm{LB}_\pi}{\mathrm{UB}_\varnothing}-1\right\},\
\frac{\mathrm{UB}_\pi}{\mathrm{LB}_\varnothing}-1\right].
\]
\end{proposition}
\begin{proof}
Deleting token edges that cross a cut reduces the feasible set without changing the
objective. Further cuts can only increase the constrained optimum. Since the corpus is
nonempty, both optima are positive. The inequalities
$\mathrm{LB}_r\le\OPT_r\le\mathrm{UB}_r$ for $r\in\{\pi,\varnothing\}$ imply
\[
\frac{\mathrm{LB}_\pi}{\mathrm{UB}_\varnothing}
\ \le\ \frac{\OPT_\pi}{\OPT_\varnothing}
\ \le\ \frac{\mathrm{UB}_\pi}{\mathrm{LB}_\varnothing}.
\]
Subtracting one and intersecting with $[0,\infty)$ proves the interval. No assumption on
LP integrality or optimiser convergence is required.
\end{proof}
The result compares cut sets at fixed $K$; it does not imply monotonicity in the vocabulary
budget.

\paragraph{General selection budgets.}
The vocabulary constraints enter the certificate through a linear maximisation oracle.

\begin{proposition}[A linear oracle for the budget]
\label{prop:anybudget}
Let $\mathcal{F}$ be a nonempty family of admissible vocabularies and let
$P\subseteq[0,1]^{\mathcal{C}}$ contain
$\{\mathbf{1}_S:S\in\mathcal{F}\}$. If $\OPT_{\mathcal{F}}$ is the minimum token count over
$\mathcal{F}$, then for every $\lambda\ge0$,
\[
\OPT_{\mathcal{F}}\ \ge\ \sum_n\SP_n(1+\lambda)
-\sup_{y\in P}\sum_t\Lambda_ty_t.
\]
If $P$ is a polytope, the maximum over $\lambda\ge0$ equals the optimum of the relaxation
with unit flows, $x_e\le y_{t(e)}$, and $y\in P$.
\end{proposition}
\begin{proof}
For any admissible segmentation $(x,\mathbf{1}_S)$, weak duality gives
\[
\sum_e x_e
\ \ge\ \sum_n\SP_n(1+\lambda)-\sum_t\Lambda_t\mathbf{1}_S(t)
\ \ge\ \sum_n\SP_n(1+\lambda)-\sup_{y\in P}\Lambda^\top y.
\]
Taking the minimum over admissible segmentations proves the bound. For a polytope $P$,
the corresponding primal relaxation is feasible and bounded. Dualising its linking
constraints and applying LP strong duality gives the equality and attainment, as in
Theorem~\ref{thm:nogap}.
\end{proof}

An outer relaxation of admissible selections preserves the bound; excluding an
admissible vocabulary need not. For the inequality, convexity is unnecessary because
a linear objective has the same supremum on $P$ and its convex hull.
The cardinality polytope
$P=\{y\in[0,1]^{\mathcal C}:\sum_t y_t\le K\}$ recovers the top-$K$ oracle.
The boundary-licence model uses two linked selection variables per token, with the
same dual argument applied to its joint selection polytope
(Appendix~\ref{app:boundary-proof}).

%% file: boundary_proofs.tex
\subsection{Occurrence-aware crossing licences}
\label{app:boundary-proof}
For every multibyte string $t$, introduce vocabulary and licence variables
$y_t,z_t\in\{0,1\}$, with $z_t\le y_t$,
$\sum_t y_t\le K$, and $\sum_t z_t\le q$.
Each document supplies a unit flow in its complete substring DAG.
For an occurrence $o$ of $t$, impose $x_o\le y_t$ if it crosses no cut and
$x_o\le z_t$ otherwise. Free byte edges ensure feasibility.
This integer formulation is equivalent to the vocabulary-and-licence
problem by the same path extraction argument as the base formulation.

\paragraph{Endpoints and nesting.}
At $q=0$, every $z_t$ is zero, so all crossing occurrences are unavailable;
the remaining paths are exactly the concatenations of pretoken paths.
At $q=K$, any unrestricted vocabulary can set $z_t=y_t$, so all its
occurrences are legal. Conversely, every licensed vocabulary is an
unrestricted vocabulary after its occurrence restrictions are removed.
Increasing $q$ leaves each previous pair $(V,W)$ feasible, proving
monotonicity. Noninteger requested fractions use $q=\lfloor fK\rfloor$;
$q$ is clipped to $[0,K]$.

\paragraph{Certificate.}
Relax the occurrence-to-vocabulary or occurrence-to-licence inequalities
with multipliers $\lambda_o\ge0$. For a feasible solution, adding
$\lambda_o(x_o-y_t)$ or $\lambda_o(x_o-z_t)$ cannot increase its cost.
Minimisation over paths gives $\sum_n\mathrm{SP}_n(1+\lambda)$.
Minimisation over $(y,z)$ gives $-H_{K,q}(A,B)$, where $A_t$ and $B_t$
sum prices on the two occurrence classes. Their difference therefore bounds
every feasible token count from below.

\paragraph{Exact support algorithm.}
Pad with zero-bid dummy types when the number of candidates is smaller than
$K$. Nonnegative bids allow an optimum with exactly $K$ selected types and
exactly $q$ licences, including unused dummy choices.
Let $m$ be the number of types after padding. Order them so $B_1\ge\cdots\ge B_m$, breaking ties by type identifier.
For a fixed selected vocabulary, exchanging a licence from a smaller-$B$
selected type to a larger-$B$ selected type cannot reduce its score, because
the $A$ contribution is unchanged. Thus some optimal assignment licences
the first $q$ selected types in this order. A split $p$ separates those
types from its $K-q$ unlicensed selected types, giving Equation~\ref{eq:boundary-split}.
Every choice counted on its right-hand side is feasible, and the exchange argument
shows that an optimum is counted. A forward and a backward heap pass
evaluate all splits in $O(m\log(K+1))$ time after sorting.
At $q=0$ the formula reduces to the top-$K$ sum of $A$; at $q=K$ it reduces
to the top-$K$ sum of $A+B$.

\paragraph{LP characterisation.}
The selection polytope is integral. To see this, set $u_t=y_t-z_t$:
its constraints are $u,z\ge0$, $u_t+z_t\le1$,
$\sum_t(u_t+z_t)\le K$, and $\sum_tz_t\le q$.
Represent each type by a capacity-one arc that splits into unlicensed
and licensed channels; the licensed channel has aggregate capacity $q$,
and their combined flow has capacity $K$. An optional-flow circulation
closes this network. Integer capacities give integral vertices, so
$H_{K,q}$ is also the exact support of this continuous selection polytope.
Dualising the occurrence-linking inequalities and applying LP strong
duality therefore identifies the maximum certificate over \emph{all}
non-negative occurrence prices with the joint flow/linking/selection LP
optimum. This equality does not imply integrality of the joint LP.
The finite optimisation and fixed hapax prices used below still yield
lower bounds; they need not attain that LP optimum.

\paragraph{An intermediate-budget integrality gap.}
Take the document \texttt{aaaaa}, cuts after bytes $1$ and $4$, $L=3$ and
$K=2$, so the candidate types are \texttt{aa} and \texttt{aaa}.
At $q=1$, either two-token path uses both types across cuts, requiring two
licences. Thus the integer optimum is $3$, attained by
\texttt{a}+$\,$\texttt{aaa}+$\,$\texttt{a}. The LP sets both vocabulary
variables to $1$, both licence variables to $1/2$, and sends half the flow
along each of \texttt{aa}+$\,$\texttt{aaa} and
\texttt{aaa}+$\,$\texttt{aa}. Its cost is $2$, the minimum path length in
the full graph, so the LP optimum is exactly $2$.
At $q=0$, both optima are $3$; at $q=2$, both are $2$.
Exact rational constraint checks and exhaustive vocabulary/licence
enumeration verify all three cases. This example establishes a possible
licensing-induced gap; it does not attribute the production intervals to
integrality rather than finite price optimisation or primal search.

\paragraph{Fixed hapax background.}
Every string that occurs once has a single occurrence, whose length and
crossing status determine its bid when its price is fixed at $\ell-1$.
Its priced edge then costs $\ell$, equal to its byte fallback.
Equal-length hapaxes with equal crossing status therefore form identical
items for the support computation, although they remain distinct vocabulary
types. No support maximiser selects more than $K$ items from one class, so
retaining $\min(K,m)$ identical representatives from a class of multiplicity
$m$ preserves the support exactly. The full graph still includes every
hapax occurrence. Restricting primal search to repeated strings affects
search quality, not the validity of an achieved vocabulary or the
certificate over the complete candidate class.

\paragraph{Witness sharing.}
For a fixed price vector, $H_{K,q}$ is nondecreasing in $q$, hence its lower
bound is nonincreasing. Taking the maximum over the same finite set of
price witnesses at every budget preserves this order. All reported
certified points can therefore share the full price pool without smoothing
or interpolating experimental values. Similarly, a vocabulary feasible at
a smaller budget remains an admissible warm start at every larger budget.

\paragraph{Certified recovery of endpoint compression gain.}
Let $[L_j,U_j]$ enclose the optimal token count at budget $j$.
When $L_0>U_K$, a positive endpoint gain is certified, and the fraction
$F_q=(\mathrm{OPT}_0-\mathrm{OPT}_q)/(\mathrm{OPT}_0-\mathrm{OPT}_K)$ obeys
\begin{equation}
\max\!\left(0,\frac{L_0-U_q}{L_0-L_K}\right)
\le F_q\le
\min\!\left(1,\frac{U_0-L_q}{U_0-U_K}\right).
\label{eq:boundary-recovery}
\end{equation}
On the nested domain $x_0\ge x_q\ge x_K$, the ratio
$(x_0-x_q)/(x_0-x_K)$ is nondecreasing in $x_0$ and $x_K$ and nonincreasing in
$x_q$. The lower corner has this ordering when $U_q<L_0$; otherwise
clipping at zero is valid. The upper corner has this ordering when
$L_q>U_K$; otherwise clipping at one is valid. Both denominators are
positive under the stated condition. The endpoint identities give
$F_0=0$ and $F_K=1$. Without $L_0>U_K$, these certificates do not identify
a positive optimal endpoint gain, and no recovery fraction is reported.
Displayed interval endpoints are rounded outwards using rational arithmetic.

%% file: setup_full.tex
\subsection{Data and constraints}
\label{app:data}
The main compression grid uses the 20231101 Wikipedia dumps in twelve languages
(en, de, ru, fi, tr, vi, ar, hi, th, ko, ja, zh). Each cell takes a prefix of
approximately $42$\,MB after retaining document lines of at least $16$ bytes.
Script arguments named \texttt{mb} use MiB; table sizes use the actual byte counts.
All comparisons preserve document boundaries and use $K=32768$ non-byte entries
with maximum length $L=16$ bytes. The $256$ single-byte entries are always available
outside this budget. Sources are UTF-8; a byte-limited prefix may retain an incomplete
final character. Regex offsets replay replacement decoding for that tail, while graph
construction and scoring preserve the original bytes. The corpus contains no byte
\texttt{ff}, which the candidate index reserves as an internal separator.

The constrained regime uses an o200k-style regex. Character offsets are converted to
byte offsets before graph edges are removed. The unrestricted regime has no internal
cuts. This definition applies to compression certification and scoring. Language-model
vocabulary construction and stream encoding are specified separately in
Table~\ref{tab:lmdecoder}. The boundary-licence experiment retains the original document
graph and checks each occurrence against the same cut positions.

\subsection{Baseline construction and scoring}
\label{app:baselines}
BPE is trained on the corresponding cut documents with Hugging Face Tokenizers.
Unigram uses SentencePiece with identity normalisation, full character coverage,
byte fallback, and no script or digit splitting. Its character-based length cap is
adapted to the corpus; after requesting an oversized vocabulary, we retain the
highest-scoring distinct byte strings satisfying $L$ until the budget is met or the
candidate pool is exhausted. The artefact records achieved vocabulary sizes.
Long training documents are split at at most $64$\,KiB, preferring whitespace and
preserving UTF-8 character boundaries; all input bytes are accounted for.

\textbf{Two-stage BPE} first learns a regex-constrained vocabulary and then performs
merges over that token stream without internal pre-tokenisation. The transition uses
half the non-byte vocabulary budget. This merge curriculum is inspired by
SuperBPE and BoundlessBPE, with a broader stage-two merge space than whitespace-only
relaxation. The separate $200$k LM experiment
includes the released SuperBPE tokeniser with its serialised inference pipeline.

Compression scores are recomputed from hexadecimal vocabularies using an independent
byte trie. Minimum-count decoding computes a shortest path under the requested cuts;
greedy longest match is a separate decoder. Native BPE ranked merges are used in the
BPE LM streams, as recorded in Table~\ref{tab:lmdecoder}. Rescoring preserves eligible
rare entries even when the optimiser represents unique substrings using shared classes.
For \slot{}, the released-vocabulary count is checked against the candidate-index count.

\subsection{Price and vocabulary optimisation}
\label{app:search}
\input{algorithm}
A suffix array and longest-common-prefix array identify all eligible candidate occurrences.
The active set contains the $M$ candidates with largest $n_t(|t|-1)$, an upper bound on
single-type token savings. The default active-set budget is $M=\max(128K,3\cdot10^6)$, capped at the
number of repeated candidates. Projected ascent uses momentum $0.9$, step size
$4/\sqrt{1+k}$, and a $1/n_t$ preconditioner. A running average is evaluated every
$25$ iterations; the largest evaluated bound is retained. Cell records give exact
iteration counts and configurations. Inactive occurrences remain present with background
prices. The independent integer checker produces the reported conservative lower bound
from the saved price vector.

The English unrestricted witness uses $M=256K=8{,}388{,}608$ and an additional
$2400$ ascent steps from a saved price vector, with step size
$0.30/\sqrt{1+k/800}$, momentum $0.9$, six CPU threads and no iterate averaging.
Its input witness and final integer check are recorded separately.

Vocabulary construction starts with greedy additions, then proposes swaps and
ruin-and-recreate moves. If $f(s)$ and $h(s)$ are forward and backward path costs under
the current vocabulary, and $c$ is the document's count, insertion of a single occurrence
$e=(s,s+\ell)$ has gain
\[
\delta_e=\max\{0,c-f(s)-1-h(s+\ell)\}.
\]
The sum of these gains over a token's occurrences ranks proposals. Overlap makes this
sum a heuristic for adding a whole token type. Removal scores use local replacement
covers. Each complete proposed vocabulary is evaluated with a shortest-path pass;
only strict improvements are retained, preserving an upper bound at every step.
Additional English search merges existing vocabularies before pruning and perturbs
low replacement-cost entries. Its final vocabulary is independently rescored.

\subsection{Validation and resource accounting}
\label{app:validation}
Table~\ref{tab:validate} compares the certificate with small LP and integer programmes.
The integer checker also has exhaustive reference tests for shortest paths, budgets,
background pricing, quantisation and overflow handling. Candidate identities and occurrence
counts are validated before evaluating production cells, and input hashes identify the
corpus and prices. The certificate is specific to those inputs.

The LM measurements draw on an inventory of $\datnLmRuns$ runs, including rate sweeps
and controls rather than independent replications of one contrast. The paired main test
result has $\datTestPairN$ seed pairs across twelve languages. The independent tuning/evaluation replication is reported in Appendix~\ref{app:gpuplan}.

%% file: algorithm.tex
\begin{algorithm}[t]
\caption{\slot{}: vocabulary construction and certification}
\label{alg:main}
\begin{algorithmic}[1]
\Require Corpus $\mathcal D$, length cap $L$, budget $K$, active set size $M$, iteration budget $T$, $K\le M\le|\mathcal{C}|$
\State Build the suffix array and longest-common-prefix index; assign candidate identities and counts $n_t$
\State Select active set $\mathcal A$ by the $M$ largest values of $n_t(|t|-1)$
\State Find $\theta_0\ge0$ by one-dimensional search of the uniform-price bound
\State Initialise active prices $\lambda_e\gets\theta_0/n_{t(e)}$, $v_e\gets0$, and $\mathrm{LB}\gets-\infty$
\For{$k=0,\dots,T-1$}
  \State Accumulate active bids $\Lambda_t$; choose top-$K$ set $S$ and its indicator $s$
  \State $\theta\gets\Lambda_{(K)}$ over $\mathcal A$ if $K<|\mathcal A|$; otherwise $\theta\gets0$
  \State Compute $\SP$ and paths $x^\star$ using active costs $1+\lambda_e$ and inactive costs $1+\theta/n_t$
  \If{$\SP-\sum_{t\in S}\Lambda_t>\mathrm{LB}$}
    \State Store this value as $\mathrm{LB}$ and save its active price vector as $\lambda^{\rm best}$
  \EndIf
  \State On active occurrences, update $v_e\gets0.9v_e+(x_e^\star-s_{t(e)})/n_{t(e)}$
  \State $\lambda_e\gets\max\{0,\lambda_e+4v_e/\sqrt{1+k}\}$; update the running average
  \State Every $25$ iterations, evaluate the average and retain it if its certificate improves $\mathrm{LB}$
\EndFor
\State Grow a vocabulary using summed single-occurrence insertion scores
\State Apply swaps and ruin-and-recreate; retain the best exactly scored feasible vocabulary $V$
\State Quantise $\lambda^{\rm best}$ and independently check the full graph and rational dual value $C_K$ (Proposition~\ref{rem:float})
\State \Return $\lceil C_K\rceil$, $V$, the checked prices, and input hashes
\end{algorithmic}
\end{algorithm}

%% file: compression_supplement.tex
\subsection{Complete language scores}
\label{app:complete}
Table~\ref{tab:compression_all} reports all twelve languages under the common headline
protocol. An upper bound on achievable bytes per token is the corpus byte count divided
by the conservative token-count lower bound. The achieved values are corpus bytes divided
by the independently scored token count. Baselines are compared at their actual
length-eligible vocabulary sizes, which cannot exceed the common budget.
\input{tables/tab_compression_all}

\subsection{Small-instance validation}
\label{app:small}
The comparisons in Table~\ref{tab:validate} distinguish the relaxation gap from price
optimisation and vocabulary-search error. The corresponding ratio factors multiply;
the percentage gaps are not an additive decomposition.
\input{tables/tab_validate}

\subsection{Budgets, regexes and deployed tokenisers}
\label{app:sweeps}
Table~\ref{tab:ksweep} repeats the English and Chinese measurement across vocabulary
budgets; Table~\ref{tab:regexes} replaces the boundary regex by deployed pre-tokenisers;
Table~\ref{tab:deployed} compares native and minimum-count decoding of released vocabularies;
it reports achieved compression without an optimality claim.
Bounds and interval endpoints in Tables~\ref{tab:ksweep} and~\ref{tab:regexes} are
rounded outwards from exact rational values, using the same convention as the prose.
\input{tables/tab_ksweep}
\input{tables/tab_regexes}
\input{tables/tab_deployed}

\subsection{Algorithmic ablations}
\label{app:ablations}
Table~\ref{tab:ablate} separates price optimisation, active-set size and vocabulary search
using the recorded English and Chinese ablations. Preconditioning and momentum improve
the dual objective at the fixed iteration budget. Iterate averaging makes no measurable
difference in these runs. Larger active sets improve the objective with diminishing
returns, while local swaps and ruin-and-recreate improve the feasible vocabulary. These
optimiser diagnostics complement the integer-checked headline certificates; they do not
supply the certified optimal-tax intervals.
\input{tables/tab_ablate}

%% file: tables/tab_compression_all.tex
\begin{table}[t]
\centering
\small
\renewcommand{\arraystretch}{1.12}
\setlength{\tabcolsep}{4pt}
\caption{\textbf{Compression scores and optimal-tax intervals.} All entries except the last column are achieved bytes per token under minimum-count encoding. Bold marks the highest achieved value within each language and regime. The final column is the outward-rounded certified optimal regex tax (\%). Two-stage denotes the local BPE curriculum. Every language uses its own Wikipedia corpus and a common budget of $32768$ non-byte entries; no cross-language byte-rate comparison is implied.}
\label{tab:compression_all}
\begin{tabular}{lrrrrrrr}
\toprule
& \multicolumn{2}{c}{Regex} & \multicolumn{4}{c}{None} & Tax (\%)\\ \cmidrule(lr){2-3}\cmidrule(lr){4-7} Language & \slot{} & BPE & \slot{} & BPE & Two-stage & Unigram & Interval\\
\midrule
English & \textbf{4.616} & 4.551 & \textbf{5.958} & 5.697 & 5.898 & 5.735 & 28.3--36.8\\
German & \textbf{4.823} & 4.714 & \textbf{6.031} & 5.641 & 5.845 & 5.622 & 24.0--33.3\\
Russian & \textbf{7.203} & 7.007 & \textbf{9.012} & 8.422 & 8.459 & 8.288 & 24.1--37.1\\
Finnish & \textbf{4.939} & 4.792 & \textbf{6.190} & 5.794 & 5.891 & 5.729 & 24.0--35.8\\
Turkish & \textbf{4.957} & 4.863 & \textbf{6.330} & 5.776 & 6.077 & 5.928 & 26.8--34.6\\
Vietnamese & \textbf{5.005} & 4.984 & \textbf{7.819} & 6.817 & 7.620 & 7.259 & 56.0--62.2\\
Arabic & \textbf{7.316} & 7.245 & \textbf{8.735} & 8.383 & 8.402 & 8.091 & 18.9--29.9\\
Hindi & \textbf{9.037} & 8.321 & \textbf{11.419} & 9.919 & 10.278 & 9.874 & 25.8--37.8\\
Thai & \textbf{10.022} & 9.153 & \textbf{11.375} & 10.279 & 10.261 & 9.854 & 9.8--23.4\\
Korean & \textbf{5.115} & 5.023 & \textbf{6.379} & 6.140 & 6.158 & 5.819 & 24.3--30.7\\
Japanese & \textbf{5.431} & 5.341 & \textbf{6.506} & 6.324 & 6.233 & 5.988 & 18.9--22.7\\
Chinese & \textbf{4.582} & 4.545 & \textbf{5.045} & 4.978 & 4.967 & 4.787 & 9.8--11.0\\
\bottomrule
\end{tabular}
\end{table}

%% file: tables/tab_validate.tex
\begin{table}[t]
\centering
\footnotesize
\renewcommand{\arraystretch}{1.08}
\setlength{\tabcolsep}{4pt}
\caption{\textbf{Checked dual values recover the full LP reference closely.} $B$ is corpus bytes and $d$ is document length in bytes. Uniform and occurrence-price bounds are evaluated exactly at dyadic prices; recovery percentages use the unrounded rational lower bound and are rounded down. The occurrence-price result is the better of two 4000-step searches. LP and integer optima are the reference HiGHS~\citep{highs2018} solves, with numerical solver tolerances. Below, relaxation, primal and dual gaps are $100(\mathrm{OPT}/\mathrm{LP}-1)$, $100(\mathrm{UB}/\mathrm{OPT}-1)$ and $100(\mathrm{LP}/\mathrm{LB}-1)$; their ratio factors multiply.}
\label{tab:validate}
\begin{tabular}{rrrrrrrr}
\toprule
$B$ & $d$ & $L$ & $K$ & $|\mathcal C|$ & LP & Uniform (\%) & Occurrence (\%)\\
\midrule
1000 & 100 & 6 & 20 & 3k & 660.0 & 91.6 & 99.99\\
1000 & 50 & 8 & 30 & 5k & 570.0 & 86.3 & 99.79\\
1500 & 100 & 10 & 30 & 9k & 880.0 & 86.5 & 99.82\\
2000 & 100 & 6 & 40 & 5k & 1167.5 & 86.1 & 99.97\\
2000 & 100 & 8 & 40 & 9k & 1124.2 & 85.9 & 99.90\\
2000 & 200 & 6 & 30 & 6k & 1256.0 & 88.4 & 99.89\\
3000 & 100 & 8 & 60 & 13k & 1597.7 & 83.4 & 99.92\\
3000 & 150 & 12 & 50 & 24k & 1648.5 & 83.7 & 99.93\\
4000 & 100 & 6 & 60 & 10k & 2202.0 & 84.2 & 99.89\\
4000 & 150 & 8 & 100 & 16k & 1774.1 & 79.6 & 99.61\\
8000 & 150 & 8 & 100 & 29k & 3847.0 & 81.2 & 99.77\\
8000 & 200 & 6 & 200 & 17k & 3093.2 & 80.1 & 99.69\\
16000 & 200 & 6 & 200 & 27k & 6532.0 & 81.4 & 99.74\\
32000 & 200 & 8 & 400 & 85k & 10253.7 & 78.0 & 99.59\\
\midrule
\multicolumn{8}{l}{\textbf{Decomposition on completed integer solves}}\\
$B$ & $d$ & $L$ & $K$ & OPT & Relaxation (\%) & Primal (\%) & Dual (\%)\\
\midrule
1000 & 100 & 6 & 20 & 660 & 0.00 & 0.45 & 0.01\\
1000 & 50 & 8 & 30 & 570 & 0.00 & 1.05 & 0.21\\
1500 & 100 & 10 & 30 & 881 & 0.11 & 1.02 & 0.18\\
2000 & 100 & 6 & 40 & 1175 & 0.64 & 2.13 & 0.02\\
2000 & 100 & 8 & 40 & 1125 & 0.07 & 0.53 & 0.09\\
2000 & 200 & 6 & 30 & 1261 & 0.40 & 0.08 & 0.10\\
3000 & 100 & 8 & 60 & 1599 & 0.08 & 0.00 & 0.07\\
3000 & 150 & 12 & 50 & 1649 & 0.03 & 1.33 & 0.06\\
4000 & 100 & 6 & 60 & 2212 & 0.45 & 1.85 & 0.10\\
\bottomrule
\end{tabular}
\end{table}

%% file: tables/tab_ksweep.tex
\begin{table}[t]
\centering
\small
\renewcommand{\arraystretch}{1.08}
\caption{Vocabulary-budget sweep at $L{=}16$, $42$\,MB per language. $B/K$ is corpus bytes per vocabulary slot; bound columns are certified ceilings in bytes per token. Achieved tax and its certified two-sided interval use Proposition~\ref{prop:tax}. BPE gaps are listed for regex and regex-free regimes. `ws' is a whitespace-only constraint at the indicated budget.}
\label{tab:ksweep}
\begin{tabular}{llrrrrcr}
\toprule
lang & $K$ & $B/K$ & bound $\pi$ & bound $\varnothing$ & achieved $\tau$ & certified $\tau$ & BPE gap $\pi\,/\,\varnothing$\\
\midrule
\multicolumn{8}{l}{\textbf{English}}\\
 & 4096 & 10240 & 3.699 & 4.542 & 10.7\% & 4.6--30.0\% & 9.8\%\,/\,23.7\%\\
 & 8192 & 5120 & 4.056 & 5.134 & 15.1\% & 11.4--30.7\% & 6.6\%\,/\,19.7\%\\
 & 32768 & 1280 & 4.642 & 6.314 & 29.1\% & 28.3--36.8\% & 2.1\%\,/\,10.9\%\\
 & 131072 & 320 & 4.994 & 8.006 & 54.5\% & 53.6--61.3\% & 1.3\%\,/\,11.9\%\\
 & \emph{ws} 32768 & 1280 & 5.058 & 6.314 & 19.0\% & 17.8--26.1\% & 3.1\%\,/\,--\\
\midrule
\multicolumn{8}{l}{\textbf{Chinese}}\\
 & 4096 & 10239 & 3.312 & 3.513 & 4.9\% & 2.9--8.2\% & 2.8\%\,/\,4.2\%\\
 & 8192 & 5120 & 3.733 & 4.010 & 6.5\% & 5.5--8.5\% & 1.8\%\,/\,3.1\%\\
 & 32768 & 1280 & 4.593 & 5.084 & 10.1\% & 9.8--11.0\% & 1.1\%\,/\,2.2\%\\
 & 131072 & 320 & 5.499 & 6.348 & 14.7\% & 14.4--15.7\% & 1.4\%\,/\,3.1\%\\
 & \emph{ws} 32768 & 1280 & 4.966 & 5.084 & 2.1\% & 1.6--2.9\% & 1.7\%\,/\,--\\
\bottomrule
\end{tabular}
\end{table}

%% file: tables/tab_regexes.tex
\begin{table}[t]
\centering
\small
\renewcommand{\arraystretch}{1.08}
\caption{Pre-tokenisation tax under deployed regex definitions on the same English headline instance. Each constrained vocabulary and bound is computed independently. cl100k and LLaMA-3 induce identical cuts on this corpus, so their rows are one measurement and their certified intervals coincide. `$\bar\ell_\pi$' is the mean induced pretoken length in bytes.}
\label{tab:regexes}
\begin{tabular}{llrcl}
\toprule
regex (deployed in) & $\bar\ell_\pi$ & achieved & certified $\tau_\pi$ & note\\
\midrule
o200k / GPT-4o style & 5.07 & 29.1\% & 28.3--36.8\% & \\
r50k (GPT-2) & 5.18 & 26.0\% & 25.3--33.5\% & \\
cl100k (GPT-3.5/4) & 5.05 & 29.2\% & 28.5--37.0\% & \\
LLaMA-3 & 5.05 & 29.2\% & 28.5--37.0\% & $\equiv$ cl100k\\
\bottomrule
\end{tabular}
\end{table}

%% file: tables/tab_deployed.tex
\begin{table}[t]
\centering
\small
\renewcommand{\arraystretch}{1.12}
\setlength{\tabcolsep}{4pt}
\caption{Released encoders and minimum-count segmentation on approximately $126$\,MB per language. Values are achieved bytes per token (higher is better). $K$ counts retained multibyte vocabulary entries of length at most $L=16$. Native uses the released encoder; the last two columns use minimum-count segmentation with common o200k-style cuts or no cuts. Removing tokens longer than $L$ increases the unrestricted minimum token count by at most $\datdepLongBenefitMax\%$ in these measurements. No optimization bound is used in this comparison.}
\label{tab:deployed}
\begin{tabular}{llrrrr}
\toprule
Language & Encoder & $K$ & Native & Cuts & No cuts\\
\midrule
English & \texttt{o200k\_base} & 195{,}757 & 4.474 & 4.501 & \textbf{4.508}\\
English & \texttt{cl100k\_base} & 99{,}200 & 4.408 & 4.437 & \textbf{4.444}\\
Chinese & \texttt{o200k\_base} & 195{,}757 & 3.145 & 3.149 & \textbf{3.177}\\
Chinese & \texttt{cl100k\_base} & 99{,}200 & 2.144 & 2.147 & \textbf{2.164}\\
\bottomrule
\end{tabular}
\end{table}

%% file: tables/tab_ablate.tex
\begin{table}[t]\centering\small
\setlength{\tabcolsep}{4pt}
\caption{Algorithm diagnostics on unrestricted English and Chinese instances, $K=32768$, $L=16$, approximately $16.8$\,MB per language, and 400 dual iterations. Values use the recorded floating-point dual objective; they are optimization diagnostics, not independent integer certificates. The primal gap is $100(\mathrm{UB}/\mathrm{LB}-1)$ relative to the full-optimizer objective. Each setting has one recorded run on 40 CPU threads; times are seconds for the indicated optimizer or individual primal stage. Uniform-price initialization and active-set timings are not reported. The largest active sets contain 16,657,541 English and 16,777,216 Chinese types.}\label{tab:ablate}
\begin{tabular}{lrrrr}\toprule Variant & English & Time & Chinese & Time\\\midrule
\multicolumn{5}{l}{\emph{Dual objective relative to the full optimizer (\%; higher is better)}}\\
Uniform prices & 85.76 & -- & 88.91 & --\\
Without preconditioning & 85.76 & 153 & 88.91 & 134\\
Without momentum & 95.67 & 287 & 98.14 & 271\\
Without averaging & \textbf{100.00} & 314 & \textbf{100.00} & 265\\
Full optimizer & \textbf{100.00} & 305 & \textbf{100.00} & 294\\
\midrule
\multicolumn{5}{l}{\emph{Tracked types: dual objective relative to the full optimizer (\%)}}\\
$|\mathcal A|=K$ & 41.05 & -- & 32.47 & --\\
$|\mathcal A|=8K$ & 86.87 & -- & 92.48 & --\\
$|\mathcal A|=32K$ & 97.04 & -- & 99.57 & --\\
$|\mathcal A|=128K$ & 100.00 & -- & 100.00 & --\\
Largest tested $|\mathcal A|$ & \textbf{100.09} & -- & \textbf{100.02} & --\\
\midrule
\multicolumn{5}{l}{\emph{Primal--dual gap after each stage (\%; lower is better)}}\\
Greedy construction & 13.74 & 82 & 4.97 & 97\\
Local refinement & 9.05 & 145 & 2.12 & 65\\
Ruin-and-recreate & \textbf{8.19} & 382 & \textbf{1.86} & 335\\
\bottomrule\end{tabular}\end{table}

%% file: lm_full.tex
\section{Language-model protocol and supplementary evidence}
\label{app:lmfull}

\subsection{Vocabulary construction and stream decoding}
\label{app:panelmech}

Vocabulary construction, admissible occurrences during encoding, and the segmentation objective are
separate parts of a tokeniser. Table~\ref{tab:lmdecoder} records the pipelines of the LM
experiments. Training, validation, test, and downstream likelihood evaluation use the same decoder
within an arm. All byte vocabularies include singleton-byte fallbacks; the end-of-sequence (EOS)
token has zero raw-byte length.

\begin{table}[t]
\centering\small
\caption{Decoder protocol of the LM experiments. ``Regex'' in the construction column
means that vocabulary fitting restricts candidate occurrences. A regex-trained byte vocabulary does
not itself impose boundary masks at encoding time. The released SuperBPE artefact retains its
serialised tokeniser pipeline.}
\label{tab:lmdecoder}
\begin{tabular}{p{.19\textwidth}p{.24\textwidth}p{.20\textwidth}p{.24\textwidth}}
\toprule
Arm & Vocabulary construction & LM encoding boundaries & Stream/evaluation decoder\\
\midrule
BPE with regex & Native BPE, regex pieces & Serialised regex & Native ranked merges\\
BPE without regex & Native BPE, complete documents & None & Native ranked merges\\
\slot{}$_\pi$ & \slot{}, regex-constrained occurrences & None & Minimum-count Viterbi\\
\slot{} & \slot{}, unrestricted occurrences & None & Minimum-count Viterbi\\
Two-stage BPE & Locally trained two-stage merge adaptation & None & Minimum-count Viterbi\\
Unigram & Likelihood-trained byte vocabulary & None & Minimum-count Viterbi\\
Released SuperBPE & Official external $200$k vocabulary & Serialised pipeline & Native ranked merges\\
\bottomrule
\end{tabular}
\end{table}

The \slot{}/\slot{}$_\pi$ pair holds the inference rule fixed while changing the vocabulary-fitting
constraint. Recorded vocabulary-search effort also differs: English and Chinese use $350$ versus
$700$ dual steps and active-set sizes $3{,}145{,}728$ versus a cap of $4{,}194{,}304$
(unrestricted versus regex); the other languages use $700$ steps in both regimes.
The BPE pair changes both its fitting boundaries and its serialised pre-tokenisation.
Local Two-stage BPE uses a broader second-stage merge space.
Unigram uses Viterbi instead of its native likelihood segmentation.
The reproduction records include an executable decoder audit and artefact hashes.

\subsection{Test-only multilingual inference}
\label{app:testprotocol}

English LM data comes from FineWeb~\citep{fineweb} (\textsc{CC-MAIN-2024-10}), Chinese from
\texttt{chinese-fineweb-edu-v2}~\citep{chinesefinewebedu}, and the other ten languages from their
FineWeb-2 subsets~\citep{fineweb2}.
These are distinct from the Wikipedia compression benchmark. Each language's arms share raw
documents and split definitions; cross-language differences can include corpus differences.
The reference vocabularies are fitted on $32$\,MB corpus prefixes, which overlap the
validation prefixes used for rate selection. Test regions lie outside the consumed training and
fitting text. Test preparation rejects documents whose
content digest occurs earlier in the corpus; source offsets, byte totals, and split metadata are
recorded for every arm. Tokenised streams concatenate documents with EOS separators. The
training loader samples token offsets with replacement, rather than aligned raw-byte starts.

The reference Transformer has 12 layers, width 768, 12 attention heads, rotary position embeddings,
RMSNorm, and SwiGLU feed-forward layers. Input and output embeddings are untied. AdamW uses
$(\beta_1,\beta_2)=(0.9,0.95)$, $\epsilon=10^{-8}$, weight decay $0.1$, and gradient clipping at
norm $1$. Training uses \texttt{bf16}, 100 warm-up steps, and a cosine schedule ending at one tenth
of the initial rate. The panel's 800 steps each contain $786{,}432$ tokens. Its base rates are
$\{0.0006,0.0012,0.0024\}$ for the regex-trained arms and
$\{0.0003,0.0006,0.0012\}$ for the unrestricted arms. All arms share a $2048$-token context and
the same padded vocabulary size. Matching tokens therefore matches model computation within a
cell, while raw-byte exposure and the byte span of the context can differ.

The primary panel has twelve languages and three paired training seeds per language
($\datpanelSeeds$). A run is
eligible only if its architecture, token budget, vocabulary size, selected rate, encoded test stream,
and scored-token count match the specified panel geometry. The analysis reads
\texttt{test\_result.json} observations and averages repeated evaluations within each arm and
seed. It then forms a relative contrast for each seed, averages seeds within language, and weights
languages equally. Hindi and Thai lack matching selected BPE test pairs; that BPE contrast
therefore covers ten languages.

Each arm's learning rate minimises the mean of repeated validation evaluations separately for each seed.
The base grid has three candidates; boundary extensions enter only when complete at all seeds and
balanced in count between the two arms. Of the $\datpanelSelN$ recorded per-seed rate selections,
$\datpanelSelInteriorN$ sit at interior grid points; the remaining $\datpanelSelEdgeN$ edge selections
are \datpanelSelEdgeList. The edge cells are not symmetric across arms: English contributes four
(the unrestricted arm at the bottom of its rate grid and the regex arm at the top, each at seeds
$5678$ and $9012$). These edge optima leave the effect of further tuning on the English contrast and the
equal-language panel mean unresolved. Test loss never enters rate selection. The result is held-out test
evaluation after within-seed validation selection, not evaluation under independent tuning
seeds. Every finite grid score, selected test observation, and byte-denominator audit has source-path
and SHA-256 provenance in \texttt{runs/panel.json}.

Per-language uncertainty uses the standard deviation of relative paired-seed contrasts and a
two-sided Student-$t$ interval with two degrees of freedom. Holm adjustment controls multiplicity
within each contrast family, conditional on the individual tests' assumptions; it does not pool
variances across languages. The sign test instead asks whether positive language-level means occur
more often than probability one half, assuming independent signs. Its panel-level conclusion differs
from significance for an individual language. The equal-language mean interval resamples languages
$20{,}000$ times with a fixed seed, conditional on the observed per-language seed means.
This language-resampling summary does not propagate training-seed or rate-selection uncertainty
and is not a prediction interval for a new language. Training-seed variability includes numerical reproducibility effects; these are
not added again as a separate variance component.

The primary analysis normalises each loss by the exact byte sum of its scored target-token
prefix, with EOS assigned zero bytes. Full-stream totals are checked against preparation metadata.
Complete token blocks cover tokeniser-dependent prefixes of the shared test slice, so exact
normalisation does not imply identical scored text.
Retained test summaries use the same target-byte convention; the validation curves in
Figures~\ref{fig:scale}--\ref{fig:conv} retain the original split-average bytes-per-token normalisation.

\input{tables/tab_test_panel}

\subsection{Learning-rate and training-scale controls}
\label{app:lrvar}
\label{app:convextok}
\label{app:xxlfull}

\begin{table}[t]
\centering
\small
\renewcommand{\arraystretch}{1.12}
\setlength{\tabcolsep}{4pt}
\caption{\textbf{The prediction penalty under scale and budget controls.}
Contrasts of \slot{} against the indicated reference arm, from the paired $\datTestN$-language
$85$M panel to the $1.88$B English rung. Test bits-per-byte entries are relative increases; BLiMP entries are accuracy-point
differences. The summary column reports the standard deviation across the
$\datTestN$ per-language mean contrasts (row one); the $1.88$B summary concerns whether the deficit continues to attenuate,
using a descriptive band twice the root-sum-square of the arm checkpoint standard deviations. Replicates are training seeds except for the $1.88$B row,
which counts checkpoints. The
$996$M selected-rate test entry uses one seed per arm, and the $387$M matched-bytes contrast's
budget response is unresolved at \datbbSeedN{} seeds per cell
(Appendix~\ref{app:lmfull}). The $387$M matched-token difference against \slot{}$_\pi$ is $\datbbDefTok$ points.
For the $1.88$B comparison against BPE, the $996$M reference difference is $\datyyAnchorBpe$ points.}
\label{tab:lmscale}
\begin{tabular}{llrrr}
\toprule
Setting & Contrast & Effect & Replicates & Summary\\
\midrule
$85$M, \datTestN{} languages & vs.\ \slot{}$_\pi$, test bits/byte & $+\datTestMean\%$ & \datTestSeedN & $\datTestSd$\\
$996$M, $K{=}200$k & vs.\ BPE, test bits/byte & $+\datxxlSlotVsBpe\%$ & 1 & --\\
$387$M, $K{=}200$k, matched bytes & vs.\ \slot{}$_\pi$, BLiMP & $\datbbDefBytes$ pts & \datbbSeedN & --\\
$1.88$B, $K{=}200$k, $D{=}20N$ & vs.\ BPE, BLiMP & $\datyyDefBpe$ pts & \datyyNBpe & unresolved\\
\bottomrule
\end{tabular}
\end{table}

Shared-rate comparisons can confound tokenisation with optimisation stability. In the English
reference sweep, the selected rates are $1.2$--$2.4\times10^{-3}$ for the two regex-trained
vocabularies and $6\times10^{-4}$ for the four unrestricted vocabularies. At the shared rate
$1.2\times10^{-3}$, the largest measured seed standard deviation is
$\datlmEnSdRatio\times$ the BPE seed spread at that same rate. We consequently match tuning budgets
and select rates separately; an interior grid minimum establishes local coverage, not a global
optimum over rates.

At $996$M and $K=200$k, the observed grid places every arm's minimum in the interior.
In the selected-rate test comparison, \slot{}$_\pi$ reaches $\datxxlSlotpiBpb$ bits per byte against
BPE's $\datxxlBpeBpb$, using seed $1234$ for each arm. The additional three-seed dispersion is measured at the
transported centre rate (defined below) and does not describe every selected-rate entry. \slot{}$_\pi$ selects
twice that centre rate. At twice the centre rate its $\datxxlSlotpiXtwoN$ seeds give
$\datxxlSlotpiXtwoMean\pm\datxxlSlotpiXtwoSd$ bits per byte, while BPE gives
$\datxxlBpeXtwoMean\pm\datxxlBpeXtwoSd$ over $\datxxlBpeXtwoN$ seeds. BPE's centre-rate loss is
$\datxxlBpeBpb$. Thus the within-regime ranking depends on rate tolerance as well as vocabulary
construction.

The shared FineWeb corpus contains $\datxxlCorpusGB$\,GB. Training windows are sampled with
replacement, and each arm's token budget is below its encoded corpus capacity.

\paragraph{Fixed tokens per parameter.}\label{app:co}
The complementary $D=20N$ ladder uses measured non-embedding parameters and jointly increases
model size and training tokens. The English \slot{} penalty falls from $\datcoFirst\%$ at
$\datcoFirstN$M to $\datcoLast\%$ at $\datcoLastN$M. Each rung transports the reference rate by
$d^{-1/2}$ and evaluates half and twice that rate. An expanded $\datcoCorpusGB$\,GB corpus keeps
every arm's consumption below its encoded capacity. This allocation fixes training tokens
per non-embedding parameter under replacement sampling.

\begin{figure}[t]
\centering
\includegraphics[width=.92\textwidth]{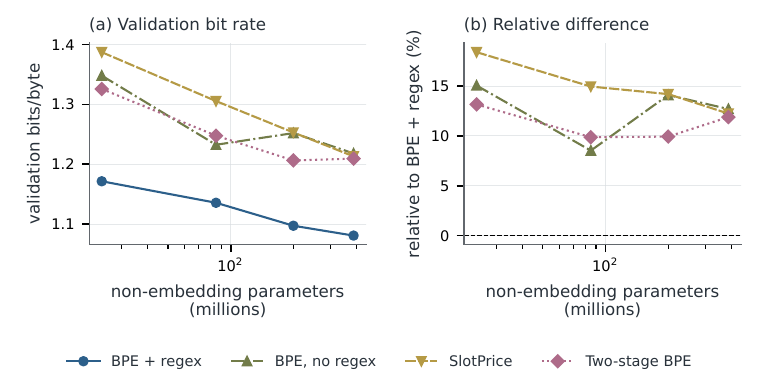}
\caption{English size ladder at $629{,}145{,}600$ training tokens ($800$ steps),
$K=32768$ non-byte entries and seed $1234$: (a) validation bits per byte and (b) loss
relative to BPE with regex. Each point uses the best validation rate for that arm and model size.}
\label{fig:scale}
\end{figure}

\begin{figure}[t]
\centering
\includegraphics[width=\textwidth]{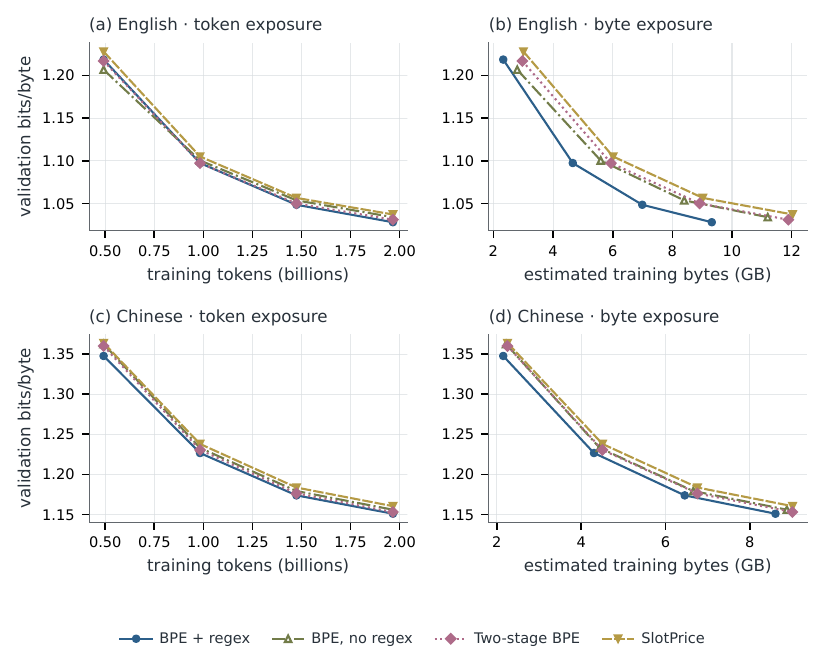}
\caption{English (top) and Chinese (bottom) validation bits per byte against training tokens
(left) and estimated byte exposure (right). All curves use an $85$M non-embedding-parameter
Transformer (12 layers, width 768, 12 heads), $K=32768$ non-byte entries ($33{,}152$ padded),
a $2048$-token context, seed $1234$ and shared initial rate $1.2\times10^{-3}$.
Exposure equals training tokens times mean training bytes per token, including repeated
sampling; losses use mean validation bytes per token. Table~\ref{tab:lmdecoder} specifies the decoders.}
\label{fig:conv}
\end{figure}

\FloatBarrier
\paragraph{Compression and prediction accounting.}
With consistent normalisation, $\mathrm{bits/byte}=\mathrm{bits/token}/\mathrm{bytes/token}$.
Thus changes in log bit rate equal changes in log bits per token minus changes in log
bytes per token. This exact accounting relation describes the measured losses; it does
not identify why a vocabulary changes prediction difficulty.

\subsection{BLiMP and byte exposure}
\label{app:downstream}

We evaluate three seeds per arm at $85$M/$32$k and $996$M/$200$k using the Benchmark of
Linguistic Minimal Pairs (BLiMP) for English~\citep{warstadt2020blimp}.
Each comparison uses raw full-sentence likelihood on $67{,}000$ minimal pairs.

The \slot{} minus \slot{}$_\pi$ BLiMP difference is $\datdsABlimpSlotVsSlotpi$ points at $85$M and
$\datdsBBlimpSlotVsSlotpi$ at $996$M. Three of four linguistic-field deficits narrow; the morphology
gap narrows from $\datdsUidAMorph$ to $\datdsUidBMorph$ points. Model size, vocabulary size,
training budget, and byte exposure all change between these settings, so this pattern does not
isolate a morphological mechanism. The $85$M-minus-$996$M paired contrast is
$\datdsAttenDiff\pm\datdsAttenDiffSd$ points (mean $\pm$ paired-seed standard deviation).
Its item-bootstrap interval
$[\datdsDoDBootLo,\datdsDoDBootHi]$ conditions on the evaluated checkpoints and does not replace
training-seed uncertainty.

\begin{figure}[t]
\centering
\includegraphics[width=\textwidth]{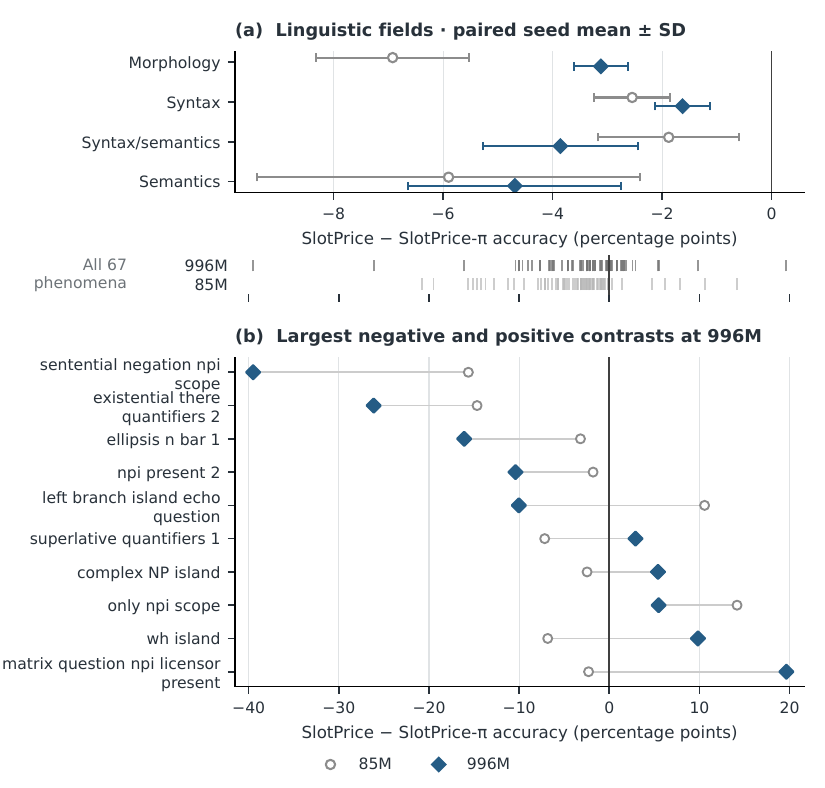}
\caption{\textbf{Paired \slot{} minus \slot{}$_\pi$ BLiMP accuracy differences.} (a) Field means and
standard deviations across paired seed contrasts. (b) The five most negative and five most
positive $996$M phenomenon-level contrasts (blue diamonds), linked to their $85$M
counterparts (open grey circles); the rug above shows all 67 phenomena at both scales,
placing the selected extrema in the full distribution. These selected extrema are exploratory.}
\label{fig:blimpuid}
\end{figure}

At $387$M/$200$k, the matched-token control uses $9835$ steps per arm; the matched-byte comparison
uses \slot{} at $6277$ steps and \slot{}$_\pi$ at $9835$. The BLiMP difference widens from
$\datbbDefTok$ to $\datbbDefBytes$ points. The longer-minus-shorter within-arm changes are
$\datbbBudgetSlot$ points for \slot{} and $\datbbBudgetSlotpi$ for \slot{}$_\pi$. With $\datbbSeedN$ seeds
per cell, their difference in budget response remains unresolved.

At $1.88$B/$200$k/$D=20N$, the \slot{} differences are $\datyyDefBpe$ points against BPE and
$\datyyDefSlotpi$ against \slot{}$_\pi$, compared with $996$M references
$\datyyAnchorBpe$ and $\datyyAnchorSlotpi$, respectively. Each arm contributes $\datyyNBpe$ checkpoints.
Twice the root-sum-square of the per-arm checkpoint standard deviations gives bands of
$\datyyBandBpe$ and $\datyyBandSlotpi$ points, respectively. These are descriptive bands, not confidence
intervals, and do not resolve whether the deficit continues to attenuate. BPE's BLiMP evaluation includes one half-rate fallback
checkpoint; its bits-per-byte dispersion uses centre-rate checkpoints. The two metrics therefore
use different checkpoint sets.

%% file: tables/tab_test_panel.tex
\begin{table}[t]
\centering
\small
\renewcommand{\arraystretch}{1.12}
\setlength{\tabcolsep}{4pt}
\caption{\textbf{Held-out test results for the paired fitted-dictionary comparison.} \slot{} dictionaries are trained with ($\pi$) or without ($\varnothing$) the regex; both use document-level Viterbi encoding. Columns show mean test bits per byte, paired relative change (\%), and its $95\%$ paired-$t$ interval over three seeds. The final column tests the twelve language contrasts using two-sided paired tests and Holm correction at $0.05$.}
\label{tab:testpanel}
\begin{tabular}{lrrrrc}
\toprule
Language & $\pi$ & $\varnothing$ & Change (\%) & $95\%$ interval & Holm\\
\midrule
English & 1.1588 & 1.3405 & +15.68 & [+13.19, +18.18] & \textbf{Yes}\\
German & 1.1773 & 1.3292 & +12.92 & [+1.78, +24.06] & No\\
Russian & 0.6245 & 0.6964 & +11.53 & [+9.12, +13.93] & \textbf{Yes}\\
Finnish & 1.2412 & 1.3619 & +9.73 & [+6.72, +12.73] & No\\
Turkish & 1.0826 & 1.1915 & +10.06 & [+3.66, +16.46] & No\\
Vietnamese & 0.8965 & 0.9428 & +5.21 & [-3.97, +14.38] & No\\
Arabic & 0.7657 & 0.8982 & +17.31 & [+8.70, +25.92] & No\\
Hindi & 0.4823 & 0.5196 & +7.77 & [+1.35, +14.18] & No\\
Thai & 0.4884 & 0.4928 & +0.93 & [-4.70, +6.57] & No\\
Korean & 1.0853 & 1.2024 & +10.78 & [+3.37, +18.19] & No\\
Japanese & 1.0393 & 1.1101 & +6.82 & [-1.40, +15.03] & No\\
Chinese & 1.1348 & 1.1607 & +2.28 & [-0.69, +5.25] & No\\
\bottomrule
\end{tabular}
\end{table}

%% file: boundary_supplement.tex
\section{Boundary-licence protocol and controls}
\label{app:boundary-protocol}
The frontier uses the project's English and Chinese Wikipedia text files.
It reads complete UTF-8 records, without partial-record truncation; following
records form a disjoint held-out compression slice. Corpus sizes, source-line
indices and SHA-256 hashes are exported in each provenance file. These records are
smaller than the $42$\,MB headline cells of Appendix~\ref{app:setup}, and the
cut-placement control uses $K=8{,}192$ rather than $32{,}768$; the reduced scale keeps
the exported candidate graphs, licence assignments and exact enumeration checks
tractable. Each table states the corpus size and budget of its own experiment. The
candidate class contains every distinct substring of length $2$ through
$L=16$ within a fitting record, including strings whose bytes are not
independently valid UTF-8. Byte fallback preserves every input byte.

The primal search selects repeated strings through insertion and reassignment
heuristics, retaining only exactly rescored improvements. The certificate
continues to cover all candidate strings. Hapax prices are fixed at
length minus one; repeated occurrences have individual non-negative prices.
The optimiser uses float64 prices. Rounding them to non-negative integer
multiples of $2^{-32}$ with round-to-nearest-even defines a new dyadic
witness, whose shortest paths and support are recomputed using checked
integer arithmetic. The ceiling of the resulting rational
lower bound is valid because the target token count is an integer.

\begin{samepage}
The reference encoder and optimisation scorer share occurrence legality and
minimum-count decoding, choosing the longest final token on ties.
IDs $0$--$255$ represent bytes; sorted multibyte strings and a separate
end-of-document ID follow. Document boundaries are never crossed. Each
vocabulary JSON records the regex, licence flags, budgets and encoding
contract. The GPU manifest binds these files by path and hash to preserve
the occurrence restriction during stream preparation.\par
\end{samepage}

CPU diagnostics include held-out bytes per token, the fraction of tokens
that cross a cut, the fraction of cuts crossed, empirical unigram entropy,
and the token mass assigned to types observed at most five times. These
describe the induced segmentation; empirical unigram entropy is not a
neural language-model loss. The random-boundary controls permute the regex
pieces' lengths in Unicode characters within each document, using a
document-hash seed. They preserve piece count and the character-length
multiset; the byte-length multiset can change.

The compression-matched controls were selected using fitting records only.
Candidate arms must share $K$ and the exact fitting-document hash, have
distinct fitting segmentations, differ by at most $1\%$ in symmetric
relative bytes per token, and differ by at least five percentage points
in their crossing-token rate measured against the same linguistic regex
cuts. The deterministic choice maximises that crossing separation and
then minimises the compression difference. Neither held-out compression
nor neural evaluation outcomes select the pairs.

Tiny instances independently enumerate vocabularies and licence assignments,
compare byte-string and graph decoders, and check integer-priced shortest
paths against exhaustive path enumeration. They include a repeated string
with both legal within-piece and crossing occurrences, exercising the case
that invalidates a fixed type partition. The exported complete graphs and
dyadic witnesses permit a separate checker to reconstruct candidates,
boundaries, the support term and the bound without invoking the optimiser.
\input{boundary_table}
\input{boundary_controls}

%% file: boundary_table.tex
\begin{table}[t]\centering\small
\caption{Boundary-licence frontier diagnostics at $K=\boundaryVocabBudget$, $L=16$. The cap $q$ and used licence count $|W|$ can differ. Fitting bytes per token are bracketed by the achieved vocabulary and certified ceiling, rounded downwards and upwards respectively. The remaining columns use disjoint held-out records totalling 1,107,737 English and 1,060,126 Chinese bytes. ``Cross'' is the percentage of emitted tokens crossing at least one regex cut; ``rare'' is token mass on types observed at most five times in that held-out slice; $H_0$ is empirical unigram entropy in bits per token. These diagnostics are not neural losses.}
\label{tab:boundary-frontier}\begin{tabular}{rrrrrrrr}\toprule
$q$ & $|W|$ & achieved & ceiling & held-out & cross (\%) & rare (\%) & $H_0$\\\midrule
\multicolumn{8}{l}{\textbf{English}}\\
0 & 0 & 4.746 & 5.059 & 4.539 & 0.0 & 11.2 & 10.56\\
655 & 655 & 5.566 & 6.274 & 5.343 & 20.0 & 13.0 & 12.04\\
1{,}638 & 1{,}638 & 5.788 & 6.509 & 5.537 & 25.1 & 13.7 & 12.43\\
3{,}276 & 3{,}276 & 5.961 & 6.711 & 5.686 & 28.5 & 15.0 & 12.74\\
8{,}192 & 8{,}192 & 6.178 & 6.958 & 5.834 & 33.0 & 17.4 & 13.14\\
32{,}768 & 32{,}768 & 6.237 & 7.221 & 5.877 & 33.5 & 18.2 & 13.20\\
\midrule
\multicolumn{8}{l}{\textbf{Chinese}}\\
0 & 0 & 4.842 & 5.056 & 4.479 & 0.0 & 14.1 & 12.20\\
655 & 655 & 5.157 & 5.435 & 4.788 & 6.4 & 14.9 & 12.72\\
1{,}638 & 1{,}638 & 5.214 & 5.488 & 4.826 & 7.2 & 15.2 & 12.79\\
3{,}276 & 3{,}276 & 5.221 & 5.509 & 4.828 & 7.5 & 15.2 & 12.80\\
8{,}192 & 3{,}276 & 5.221 & 5.510 & 4.828 & 7.5 & 15.2 & 12.80\\
32{,}768 & 32{,}768 & 5.221 & 5.510 & 4.829 & 7.8 & 15.3 & 12.80\\
\bottomrule\end{tabular}\end{table}
\begin{table}[t]\centering\small
\caption{Recovery of the endpoint token-count reduction at interior budgets. The achieved ratio compares exported vocabularies. The certified interval encloses the fraction of the optimal endpoint gain recovered at that budget, using Equation~\ref{eq:boundary-recovery}; endpoints are rounded outwards. Its upper endpoint is $100\%$, so the final column prints the certified lower endpoint. At the omitted endpoints $q=0$ and $q=K$, recovery is exactly $0\%$ and $100\%$ by definition.}
\label{tab:boundary-recovery}\begin{tabular}{lrrr}\toprule
Language & $q$ & achieved (\%) & optimal gain (\%)\\\midrule
English & 655 & 61.6 & $\geq 30.4$\\
English & 1{,}638 & 75.3 & $\geq 42.1$\\
English & 3{,}276 & 85.2 & $\geq 50.5$\\
English & 8{,}192 & 96.9 & $\geq 60.5$\\
Chinese & 655 & 84.2 & $\geq 24.1$\\
Chinese & 1{,}638 & 98.3 & $\geq 37.1$\\
Chinese & 3{,}276 & 100.0 & $\geq 38.7$\\
Chinese & 8{,}192 & 100.0 & $\geq 38.7$\\
\bottomrule\end{tabular}\end{table}

%% file: boundary_controls.tex
\begin{table}[t]\centering\small
\caption{Cut-placement control at $K=8{,}192$, $L=16$, with 2,098,205 English and 2,136,886 Chinese fitting bytes. Random cuts permute each document's regex-piece lengths in Unicode characters, preserving cut count and the character-length multiset; byte-length multisets can differ. The achieved and optimal columns give $100(\mathrm{tokens}_{\rm random}/\mathrm{tokens}_{\rm regex}-1)$ for fitting vocabularies and optima respectively. Certified interval endpoints are rounded outwards. At $q=K$ the optimal token-count change is exactly zero because both feasible sets are unrestricted, although finite-search vocabularies can differ. Crossing rates use the same linguistic regex cuts for both arms; they describe fitting segmentations, not neural prediction.}
\label{tab:boundary-controls}\begin{tabular}{rrrrrrr}\toprule
 & \multicolumn{2}{c}{Random/regex change (\%)} & \multicolumn{2}{c}{Held-out bytes/token} & \multicolumn{2}{c}{FIT crossing (\%)}\\
$q$ & achieved & optimal interval & regex & random & regex & random\\\midrule
\multicolumn{7}{l}{\textbf{English}}\\
0 & +48.7 & [47.1,53.5] & 3.803 & 2.743 & 0.0 & 26.6\\
163 & +47.1 & [15.2,57.0] & 4.169 & 3.018 & 13.2 & 26.4\\
409 & +42.9 & [14.9,52.3] & 4.282 & 3.179 & 16.9 & 27.6\\
819 & +37.6 & [12.0,46.8] & 4.360 & 3.344 & 20.0 & 28.5\\
2{,}048 & +23.7 & [1.2,35.5] & 4.365 & 3.660 & 20.3 & 29.7\\
8{,}192 & +2.4 & [0.0,0.0] & 4.365 & 4.250 & 20.3 & 28.4\\
\midrule
\multicolumn{7}{l}{\textbf{Chinese}}\\
0 & +5.4 & [3.9,6.5] & 3.445 & 3.373 & 0.0 & 5.8\\
163 & +7.6 & [5.2,9.1] & 3.616 & 3.437 & 4.2 & 6.1\\
409 & +6.7 & [4.4,8.2] & 3.646 & 3.466 & 4.9 & 6.2\\
819 & +5.3 & [3.0,6.9] & 3.648 & 3.496 & 5.1 & 6.3\\
2{,}048 & +2.8 & [0.5,4.4] & 3.648 & 3.567 & 5.1 & 6.6\\
8{,}192 & +0.3 & [0.0,0.0] & 3.650 & 3.662 & 5.2 & 7.0\\
\bottomrule\end{tabular}\end{table}
\paragraph{Compression-matched exports.} These controls use $K=8{,}192$ multibyte slots. For English, regex $q=409$ and random $q=8{,}192$ differ by 0.12\% in symmetric fitting compression and 11.46 percentage points in common-regex crossing rate. For Chinese, regex $q=0$ and random $q=819$ differ by 0.14\% in symmetric fitting compression and 6.33 percentage points in common-regex crossing rate. These pairs were selected solely from fitting compression and crossing diagnostics under the stated deterministic rule.

%% file: related_full.tex
\paragraph{Vocabulary optimisation.}
BPE~\citep{gage1994,sennrich2016} and Unigram~\citep{kudo2018,sentencepiece} optimise
surrogate vocabulary objectives. Among dictionary-based methods, \citet{galle2019} found that shorter test-set
tokenisations at a fixed vocabulary budget correlate with higher translation quality,
and Unigram vocabularies match or outperform BPE for masked-language-model
pretraining~\citep{bostrom2020}. ConvexTok~\citep{convextok} instead relaxes joint
vocabulary selection and minimum-count segmentation to a linear programme;
JOLT~\citep{jolt} adds consistency with greedy inference. Our formulation shares the
vocabulary-selection polytope with ConvexTok. Its contribution is a specialised price
certificate that can be evaluated without a general-purpose LP solve and without
internal pre-tokenisation. The full price family attains the same LP optimum, while
background pricing makes restricted optimisation possible without discarding candidates.

\paragraph{Lagrangian pricing.}
Relaxing coupling constraints to obtain separable subproblems is
classical~\citep{held1970,geoffrion1974,fisher1981}. Related cardinality and facility-location
problems also use prices and selection
subproblems~\citep{cornuejols1977,erlenkotter1978}. Here the separable problem is corpus segmentation,
and the vocabulary support function becomes a top-$K$ sum. Boundary licences replace
that support function with a nested selection problem while preserving shortest-path
evaluation.

\paragraph{Relaxing boundaries.}
SuperBPE~\citep{superbpe} and BoundlessBPE~\citep{boundlessbpe} learn tokens spanning
conventional pretoken boundaries. Efficient superword training~\citep{fastersuper} and
entropy-based boundary selection~\citep{entropypretok} explore related design choices.
These methods construct vocabularies; our optimal-tax interval measures what a boundary
constraint excludes, independently of a particular construction heuristic.

\paragraph{Token units and prediction.}
PathPiece~\citep{pathpiece} shows that lower token counts need not improve downstream
performance. R\'enyi efficiency~\citep{renyi} and its
counterexamples~\citep{counterexamples} likewise motivate evaluating prediction beyond intrinsic
compression scores. TokEval~\citep{tokeval2026} studies intrinsic and structural metrics
with controlled LM pretraining; \citet{erdogan2026} examine how tokenisation changes
entropy and sensitivity to domain shift. Our contribution is to pair downstream
measurements with certified compression optima and a budgeted intervention on boundary
crossing. Token-free architectures such as MEGABYTE, MambaByte and the Byte Latent
Transformer~\citep{megabyte,mambabyte,blt} provide an architectural alternative; the
present experiments hold the model family fixed.

%% file: checklist.tex
The supplementary archive contains four compact data tables, their reported summaries and
\texttt{verify.py}. It runs with Python 3.10 or later using only the standard library:
\begin{center}
\texttt{python verify.py}
\end{center}
The script verifies file checksums and reconstructs the relative paired-seed effects,
means, standard deviations, paired-$t$ intervals and Holm adjustments. For independent
replication it also reconstructs bits per byte from accumulated negative log-likelihoods
and exact scored-byte counts. The compression tables provide integer lower bounds and
achieved counts; rational arithmetic reconstructs the outward-rounded optimal-tax
intervals and achieved licence recovery. These checks reproduce the arithmetic from
recorded measurements. Full-corpus certification additionally evaluates the candidate
graphs and saved occurrence prices, using the procedures in
Appendices~\ref{app:price-verification} and~\ref{app:boundary-proof}.

The archive's README specifies every field, experimental panel and evaluation mode.
All data and scripts needed for these checks are included. The separate manuscript-source
archive includes the figures, tables, bibliography and official style files; it compiles
with pdfLaTeX and BibTeX. The extracted source archive is independently compiled and
compared with the delivered PDF page by page.

%% file: gpu_pending.tex
\subsection{Independent tuning and evaluation}
\label{app:gpuplan}
This replication compares regex-constrained and unrestricted dictionary construction with
separate seeds for learning-rate selection and final evaluation. Both arms use exported
byte vocabularies and document-level minimum-count Viterbi encoding. Within each language,
the arms share raw documents, model architecture, token budget, validation-selection rule
and scored test-byte interval. Learning rates minimise mean validation loss over tuning seeds $11001$ and $11002$;
final comparisons pair evaluation seeds $21001$, $21002$ and $21003$.
Table~\ref{tab:replication-rates} reports the candidate grids and selected rates. The model has
12 layers, width 768, 12 heads and a padded vocabulary of $33{,}152$ entries; each run
uses 800 steps of $786{,}432$ target tokens. Training samples common raw-byte start
positions, with tokeniser-specific endpoints under the common token budget.

Evaluation uses three context policies on the same target-byte region. Independent blocks
contain at most $2048$ tokens. Byte-capped history retains the longest whole-token suffix
within $1024$ raw bytes before each target; discarded prefixes are absent from the model
input. Sliding evaluation retains up to $2048$ tokens, advances by $128$ tokens, and scores
only the newly exposed targets. Each mode normalises by the exact bytes in scored targets;
end-of-document markers provide context but contribute no scored bytes.

\input{tables/tab_replication_rates}
\input{gpu_results}

Intervals describe variation across paired evaluation seeds after selecting rates using
the tuning seeds. Holm correction applies across the complete twelve-language family.
These measurements test the fitted-dictionary comparison under independent rate selection;
they do not evaluate intermediate boundary-licence budgets.

%% file: tables/tab_replication_rates.tex
\begin{table}[t]
\centering
\small
\setlength{\tabcolsep}{7pt}
\caption{Learning-rate grids and selected rates for the independent replication. All rates are in units of $10^{-3}$. The grids are $A=\{0.3,0.6,1.2,2.4\}$, $B=A\cup\{4.8,9.6\}$, and $H=\{0.009375,0.01875,0.0375,0.075,0.15\}$. Every grid point uses two tuning seeds. The selected rate minimizes their mean validation bits per byte and lies strictly inside its grid. Three disjoint seeds provide test measurements.}
\label{tab:replication-rates}
\begin{tabular}{lrrrr}
\toprule
& \multicolumn{2}{c}{Unrestricted} & \multicolumn{2}{c}{Regex-fitted}\\
\cmidrule(lr){2-3}\cmidrule(lr){4-5}
Language & Grid & Rate & Grid & Rate\\
\midrule
English & $A$ & 0.6 & $B$ & 2.4\\
German & $A$ & 0.6 & $B$ & 2.4\\
Russian & $A$ & 1.2 & $B$ & 2.4\\
Finnish & $B$ & 2.4 & $B$ & 2.4\\
Turkish & $A$ & 0.6 & $B$ & 2.4\\
Vietnamese & $A$ & 0.6 & $B$ & 2.4\\
Arabic & $A$ & 1.2 & $B$ & 2.4\\
Hindi & $A\cup H$ & 0.0375 & $B\cup H$ & 0.0375\\
Thai & $A\cup H$ & 0.0375 & $B\cup H$ & 0.0375\\
Korean & $A$ & 0.6 & $B$ & 1.2\\
Japanese & $A$ & 0.6 & $B$ & 1.2\\
Chinese & $A$ & 1.2 & $B$ & 1.2\\
\bottomrule
\end{tabular}
\end{table}

%% file: gpu_results.tex
\begin{longtable}{@{}p{.25\textwidth}p{.19\textwidth}p{.33\textwidth}p{.11\textwidth}@{}}
\caption{Independent tuning and evaluation: paired test contrasts. Changes are $100(\mathrm{bpb}_{\varnothing}/\mathrm{bpb}_{\pi}-1)$; positive values favor the regex-fitted dictionary. Entries give the mean and two-sided $95\%$ paired-$t$ interval over three evaluation seeds, disjoint from the two tuning seeds. Holm adjustment covers twelve languages separately within each context mode. Token blocks use 2048-token contexts; byte history uses a physical whole-token suffix of at most 1024 bytes per target; sliding contexts use strided overlapping windows.}\label{tab:gpu_verified}\\
\toprule Comparison & Context & Relative change (\%) / status & Holm $p$\\\midrule\endfirsthead
\toprule Comparison & Context & Relative change (\%) / status & Holm $p$\\\midrule\endhead
English & Token blocks & 10.44 [6.23, 14.64] & 0.060\\
Chinese & Token blocks & 4.62 [1.96, 7.29] & 0.060\\
German & Token blocks & 11.06 [8.99, 13.14] & 0.021\\
Finnish & Token blocks & 6.41 [3.85, 8.97] & 0.060\\
Russian & Token blocks & 7.45 [3.07, 11.84] & 0.060\\
Japanese & Token blocks & 3.36 [1.74, 4.97] & 0.060\\
Thai & Token blocks & -0.59 [-1.47, 0.30] & 0.105\\
Turkish & Token blocks & 8.37 [6.65, 10.09] & 0.023\\
Vietnamese & Token blocks & 3.42 [1.99, 4.85] & 0.060\\
Arabic & Token blocks & 13.69 [10.45, 16.94] & 0.027\\
Hindi & Token blocks & 10.22 [9.28, 11.16] & 0.005\\
Korean & Token blocks & 5.35 [3.51, 7.19] & 0.051\\
English & Byte history & 9.90 [6.95, 12.85] & 0.040\\
Chinese & Byte history & 4.95 [1.80, 8.10] & 0.070\\
German & Byte history & 9.61 [7.76, 11.46] & 0.022\\
Finnish & Byte history & 6.01 [4.27, 7.75] & 0.040\\
Russian & Byte history & 5.75 [1.73, 9.76] & 0.070\\
Japanese & Byte history & 2.62 [1.11, 4.13] & 0.070\\
Thai & Byte history & -0.11 [-0.95, 0.74] & 0.645\\
Turkish & Byte history & 6.88 [5.41, 8.35] & 0.025\\
Vietnamese & Byte history & 3.78 [2.57, 4.98] & 0.040\\
Arabic & Byte history & 11.13 [7.51, 14.76] & 0.040\\
Hindi & Byte history & 14.73 [12.75, 16.71] & 0.012\\
Korean & Byte history & 5.12 [3.22, 7.02] & 0.040\\
English & Sliding & 10.68 [6.19, 15.17] & 0.053\\
Chinese & Sliding & 5.36 [2.97, 7.76] & 0.053\\
German & Sliding & 11.36 [9.24, 13.48] & 0.021\\
Finnish & Sliding & 6.57 [3.82, 9.32] & 0.053\\
Russian & Sliding & 7.94 [3.40, 12.49] & 0.053\\
Japanese & Sliding & 3.16 [2.22, 4.10] & 0.038\\
Thai & Sliding & -0.39 [-1.26, 0.47] & 0.189\\
Turkish & Sliding & 8.84 [7.18, 10.51] & 0.021\\
Vietnamese & Sliding & 3.44 [2.05, 4.84] & 0.053\\
Arabic & Sliding & 14.04 [10.67, 17.41] & 0.028\\
Hindi & Sliding & 10.58 [10.04, 11.13] & 0.002\\
Korean & Sliding & 5.43 [3.68, 7.18] & 0.039\\
\bottomrule
\end{longtable}